\documentclass{article}

\usepackage[preprint]{neurips_2026}

\usepackage[utf8]{inputenc} % allow utf-8 input
\usepackage[T1]{fontenc}    % use 8-bit T1 fonts
\usepackage[hidelinks]{hyperref}       % hyperlinks
\usepackage{url}            % simple URL typesetting
\usepackage{booktabs}       % professional-quality tables
\usepackage{amsfonts}       % blackboard math symbols
\usepackage{nicefrac}       % compact symbols for 1/2, etc.
\usepackage{microtype}      % microtypography
\usepackage{xcolor}         % colors

\usepackage{multirow}
\usepackage{amsmath}
\usepackage{amsthm}
\usepackage{mathtools}

\usepackage[capitalize,noabbrev]{cleveref} 

\usepackage{algorithm}
\usepackage{algorithmic}
\usepackage{arydshln}

\usepackage{listings}
\usepackage{xspace}
\usepackage{thmtools}
\usepackage{thm-restate}
\usepackage{subfig}

\usepackage{colortbl}       % Colorful tables
\usepackage{graphicx}
\usepackage{tabularx}
\usepackage{float}
\usepackage{enumitem}
\usepackage{xparse}
\usepackage[export]{adjustbox}

\usepackage[english]{babel}

\newtheorem{theorem}{Theorem}
\newtheorem{corollary}[theorem]{Corollary}
\newtheorem{definition}[theorem]{Definition}
\newtheorem{lemma}[theorem]{Lemma}

\theoremstyle{remark}
\newtheorem{remark}{Remark}

\newcommand{\floor}[1]{{\left \lfloor #1 \right \rfloor}}
\newcommand{\ceil}[1]{{\left \lceil #1 \right \rceil}}

\newcommand{\N}{\mathcal{N}}

\newcommand{\T}{\mathcal{T}}

\usepackage[disable]{todonotes}

\newcommand{\placeholderchristian}[1]{\todo[inline, color=gray]{TODO responsible Christian: #1}}

\newcommand{\cl}[1]{}
\newcommand{\de}[1]{}
\newcommand{\tc}[1]{}
\newcommand{\ab}[1]{}

\newcommand{\methodname}{\texttt{Lumberjack}}
\title{Lumberjack: Better Differentially Private Random Forests through Heavy Hitter Detection in Trees}

\author{%
    Christian Janos Lebeda\textsuperscript{3,$\star$,}\thanks{Work done as a postdoc at PreMeDICaL, Inria.}\,\,,
	David Erb\textsuperscript{2,$\star$,}\thanks{Work done while visiting PreMeDICaL, Inria.}\,\,,
	Tudor Cebere\textsuperscript{1},
	Aurélien Bellet\textsuperscript{1}\\
\\
	\textsuperscript{$\star$}Lead authors\\
	\\
	\textsuperscript{1}PreMeDICaL, Inria, Université de Montpellier, INSERM, France\\
    \textsuperscript{2}Technical University of Munich, Germany\\
    \textsuperscript{3}University of Copenhagen, Denmark
}

\begin{document}

\maketitle

\begin{abstract}
    Random forests are widely used in fields involving sensitive tabular data, but existing approaches to enforcing differential privacy (DP) typically degrade performance to the point of impracticality. In this paper, we introduce \methodname{}, a differentially private random forest algorithm that achieves substantially higher utility by constructing large random decision trees and then applying aggressive, privacy-preserving pruning to retain only sufficiently populated nodes. A key component of our approach is a novel $(\varepsilon,\delta)$-DP heavy hitter detection algorithm for hierarchical data, whose error is $O_{\varepsilon,\delta}(\sqrt{\log h})$ for trees of height $h$ and may be of independent interest. This favorable scaling enables the use of significantly deeper trees than in prior work, leading to improved expressiveness under privacy constraints. Our empirical evaluation on benchmark datasets shows that \methodname{} consistently outperforms prior DP random forest methods, establishing a new state of the art. In particular, our approach yields substantial improvements in the privacy-utility trade-off for practical privacy budgets. Our findings suggest that carefully designed DP random forests can close much of the utility gap, highlighting a promising and underexplored direction for future research.
\end{abstract}

\section{Introduction}
\label{sec:intro}

Random forests (RFs) \citep{breiman2001random,geurts2006extremely} remain among the most accurate and most widely used methods for tabular prediction, often matching or outperforming more complex deep learning approaches \citep{fernandez2014we,DBLP:conf/nips/GrinsztajnOV22,kaggle,10.1371/journal.pone.0301541}. Built as ensembles of decision trees that recursively partition the feature space and aggregate predictions across leaves, RFs combine strong predictive performance with several practical advantages: they naturally accommodate heterogeneous feature types, require comparatively little hyperparameter tuning, and are often substantially more computationally efficient than deep neural networks \citep{DBLP:conf/nips/GrinsztajnOV22}. These properties have made them a central tool in many high-stakes applications involving sensitive tabular data, from healthcare to financial risk assessment.

At the same time, like other machine learning models, random forests can leak information about individual records used to train them. For instance, recent work has shown that, under certain conditions, some training examples can be reconstructed from publicly released random forest models \citep{pmlr-v235-ferry24a}, motivating the need for formal privacy protections. Differential Privacy (DP) \citep{DworkMNS06} has emerged as the standard framework for addressing such risks by providing rigorous guarantees that an algorithm's output changes only minimally when any one individual's data is added or removed from the training set.\looseness=-1

% Differential Privacy (DP) \citep{DworkMNS06} has emerged as the standard for mitigating these risks by providing a mathematical framework to assess the information leaked about any individual in the training dataset, ensuring that the model output remains indistiguishable whether a individual's data was included or not. Driven by the remarkable success of deep learning over the last decade, most differentially private machine learning research has focused on training deep neural networks using differentially private optimizers such as DP-SGD \citep{abadi_deep_2016}. However, as shown by \citet{tramer_differentially_2021}, the performance of neural networks is heavily penalized by noisy gradient methods, as the magnitude of the noise must scale with the dimensionality of the model's parameters. Consequently, neural networks trained via DP-SGD can be outperformed by much simpler and less-parametrized learning paradigms. 

Much of the literature on differentially private machine learning has focused on differentiable models, particularly deep neural networks trained with private optimization methods such as DP-SGD \citep{abadi_deep_2016,DBLP:journals/jair/PonomarevaHKXDMVCT23}. However, this paradigm can incur substantial utility costs under privacy constraints, because private gradient-based training injects noise into a high-dimensional optimization process, with costs that scale unfavorably with model dimensionality \citep{tramer_differentially_2021}. As a result, private deep models may underperform substantially simpler learning methods.\looseness=-1

This motivates the exploration of alternative model classes with differently structured privacy-utility trade-offs, such as random forests. 
% Consequently, a natural candidate to study in this regime is random forests. \cl{It's not just to avoid the noise of high-parametrized settings. Forests are also preferred in practice in many non-private scenarios. Aurelien or I can add some references.}
However, constructing differentially private RFs presents several challenges.
The primary technical difficulty lies in building high-quality decision tree structures while satisfying strict privacy constraints.
%privately build effective decision trees.
%The key challenge is how to design an effective privacy-preserving decision tree building algorithm. \cl{Should be more precise. It's unclear what tree building means. It's the structure that's the primary challenge for DP, not leaves}
Decision tree algorithms recursively partition the dataset by selecting a feature and split threshold value. In standard non-private approaches \citep{breiman2001random}, these splits are chosen greedily by querying the data to optimize an impurity-based criterion.
% When a stopping condition such as maximum depth or minimum node size is met, a node is converted to a leaf which can be used for future predictions.
% Standard algorithms such as in \citep{breiman2001random} query the dataset to greedily select a good split based on minimizing an impurity measure. 
When adapted naively to the differentially private setting, such data-dependent procedures incur significant noise, leading to poor utility~(e.g. \citep{PatilS14_random_forest,FletcherI15_greedy}).\footnote{Unfortunately, the existing literature on differentially private random forests also includes several incorrect or flawed analyses. We found privacy issues in more than 10 papers, which we summarize in \Cref{app:privacy-concerns}.
% For this reason it is unclear which algorithms should be considered the prior state-of-the-art. 
Despite the fact that these methods do not satisfy differential privacy as claimed, our approach still substantially outperforms their reported performance.}
A common workaround is to use fully random trees, inspired by extremely randomized trees \citep{geurts2006extremely}. In these approaches \citep{JagannathanPW12_small_random, JagannathanMP13_semi_private_data, FletcherI15_random, fletcher2017differentially,holohan2019diffprivlibibmdifferentialprivacy}, tree structures are generated using splits that are completely independent of the data, while the data is used only to label leaves.
%structure is generated independently of the data-splitting features, and thresholds are chosen uniformly at random, the data being exclusively reserved for querying the labels of the leaf nodes. 
Unfortunately, fully random trees also suffer from significant limitations. 
%Performance is typically very sensitive to the depth hyperparameter. 
Their performance is highly sensitive to tree depth: shallow trees yield overly large leaves, while deep trees produce leaves with too few samples for reliable privatized predictions. They are also strongly dependent on the underlying data distribution and often fail to capture informative structure, collapsing to majority-class predictions. We illustrate this behavior on a toy example in \Cref{fig:moons-toy-example}, using an infinite privacy budget to emphasize that this is an inherent limitation of data-independent tree construction. Consequently, DP ensembles of fully random trees typically perform poorly. Overall, existing differentially private random forest methods, including both greedy and fully random variants, are of limited practical utility.
\looseness=-1

% Unfortunately, fully random trees exhibit multiple limitations in practice. Performance is often very sensitive to the choice of depth. If the trees are not deep enough, some leaves will contain many data points. If the trees are too deep, the leaves will contain only a few data points, which can be insufficient for privatized the leaf predictions.
% Because the trees expand to a predetermined maximum depth, on the final layer, a vast amount of the leaves are wasted on empty or extremely sparse regions of the feature space. While prior work \citep{FletcherI17_smooth} notes that these empty leaves do not harm predictive performance, as new data is unlikely to fall in these sparse regions, we argue that this is not the full picture. Whenever a random split creates an empty node, we effectively "wasted" a split. These waster split have minimal impact on adaptive techniques but they can seriously hurt performance for fixed depth tree. In extreme cases the trees are unable to separate important parts of the feature space and the estimators revert to majority class estimation~(see \Cref{fig:moons-toy-example}).
% Unsurprisingly, constructing a random forest from fully data-independent trees typically leads to poor performance.

\begin{figure}[t]
    \centering
    % First minipage: 75% of the text width for the image grid
    \begin{minipage}{0.72\textwidth}
        \centering
        \setlength{\tabcolsep}{2pt} 
        \renewcommand{\arraystretch}{0}
        \begin{tabular}{ccc}
            \includegraphics[width=0.32\linewidth, valign=m]{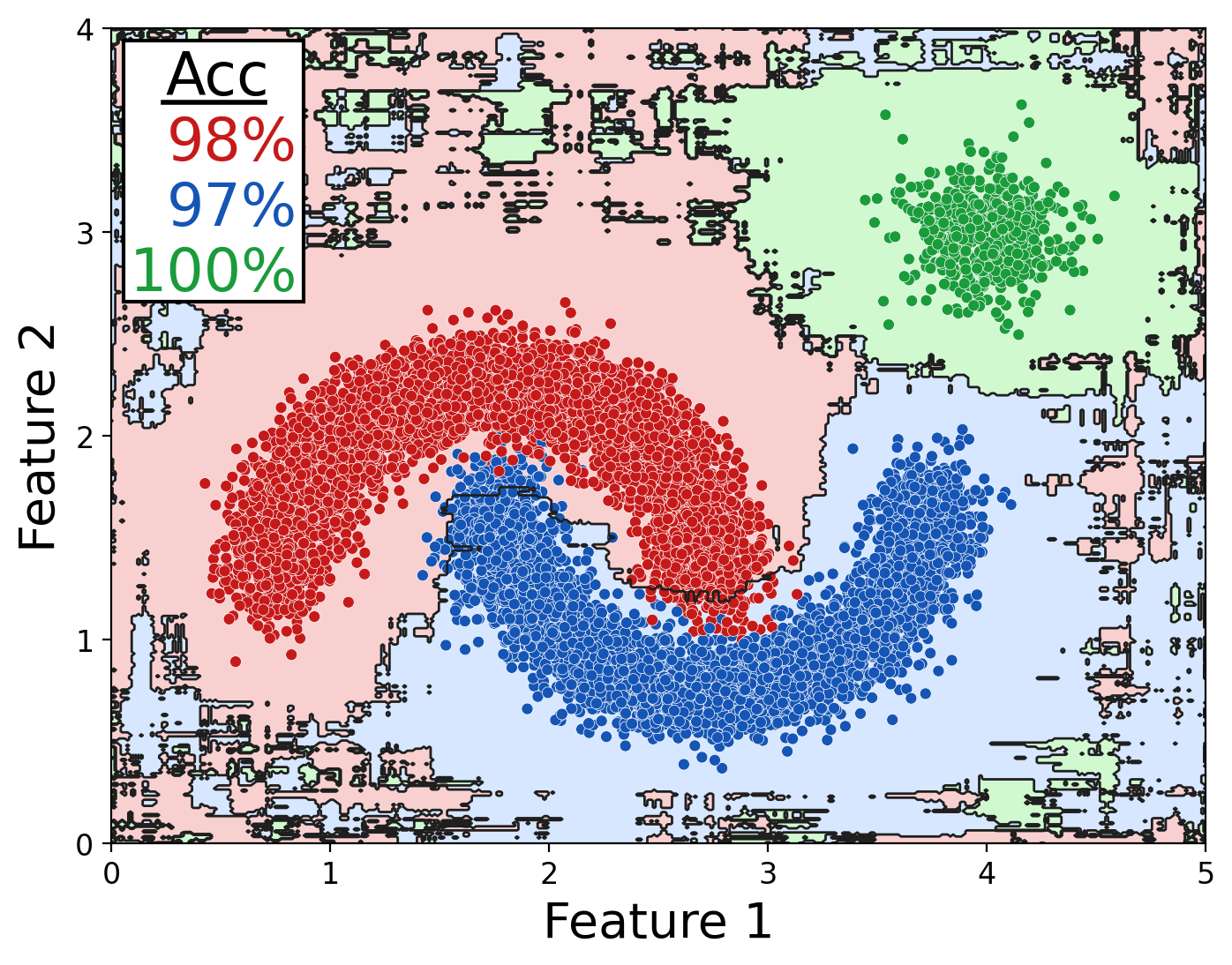} &
            \includegraphics[width=0.32\linewidth, valign=m]{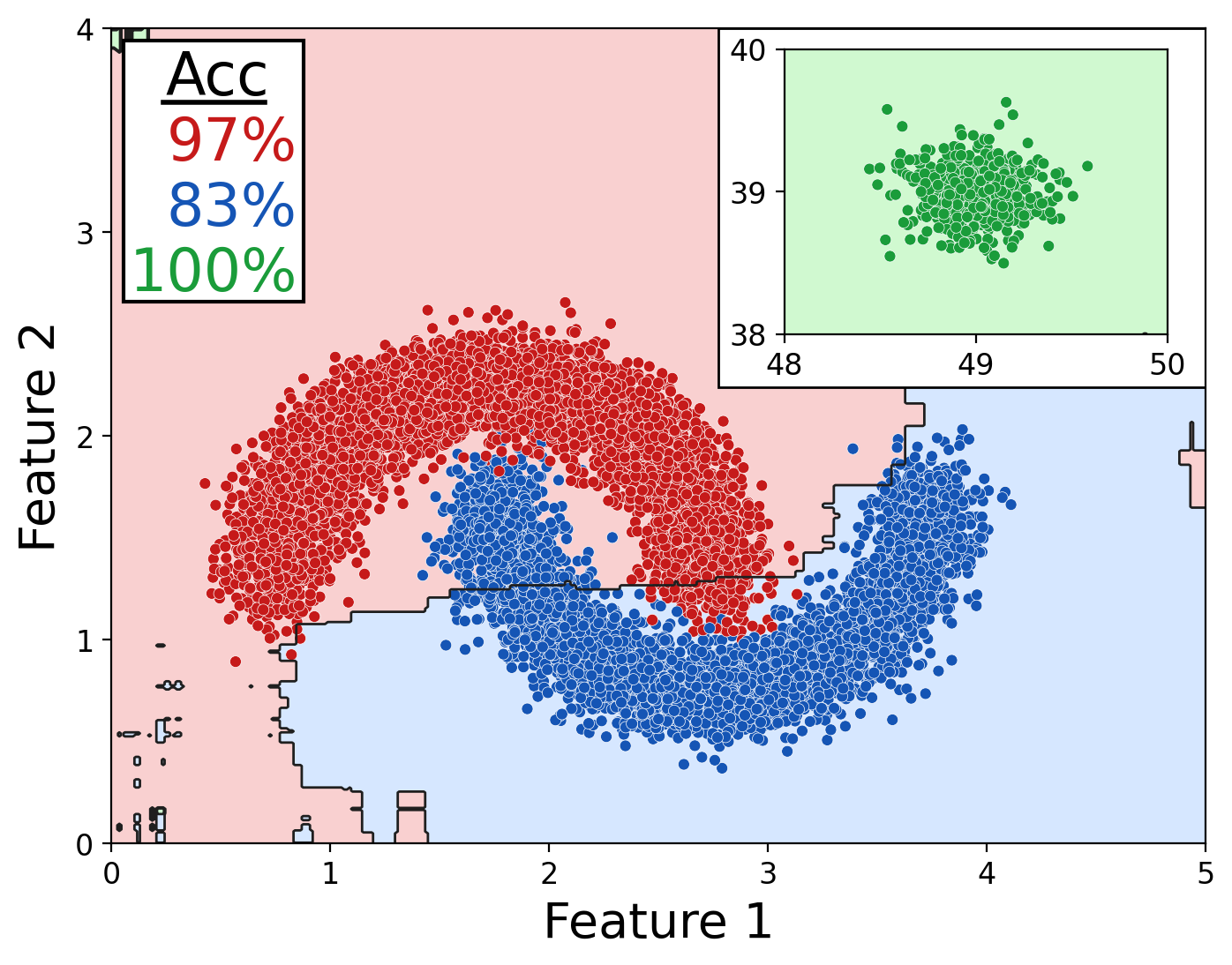} &
            \includegraphics[width=0.32\linewidth, valign=m]{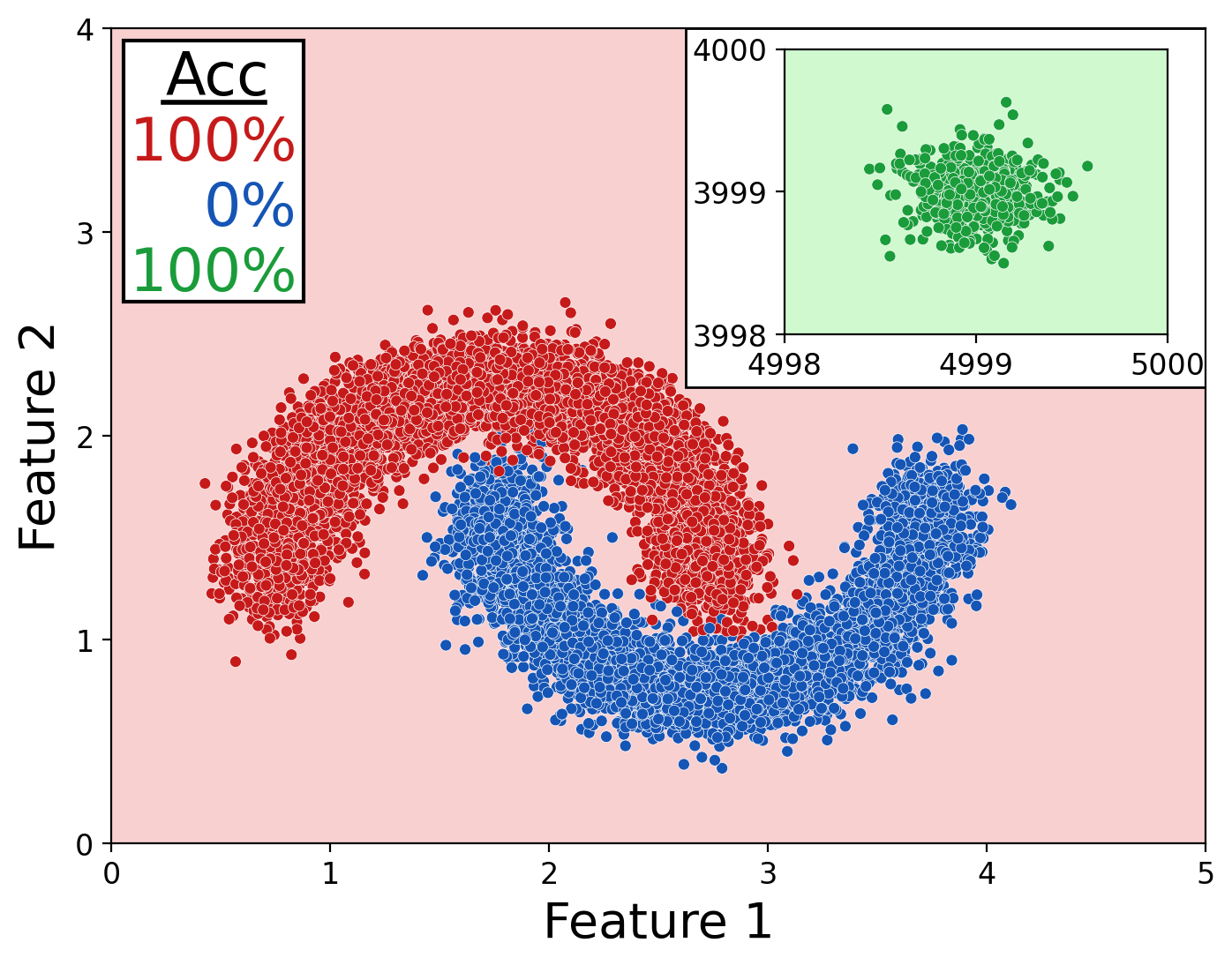} \\[4pt]

            \includegraphics[width=0.32\linewidth, valign=m]{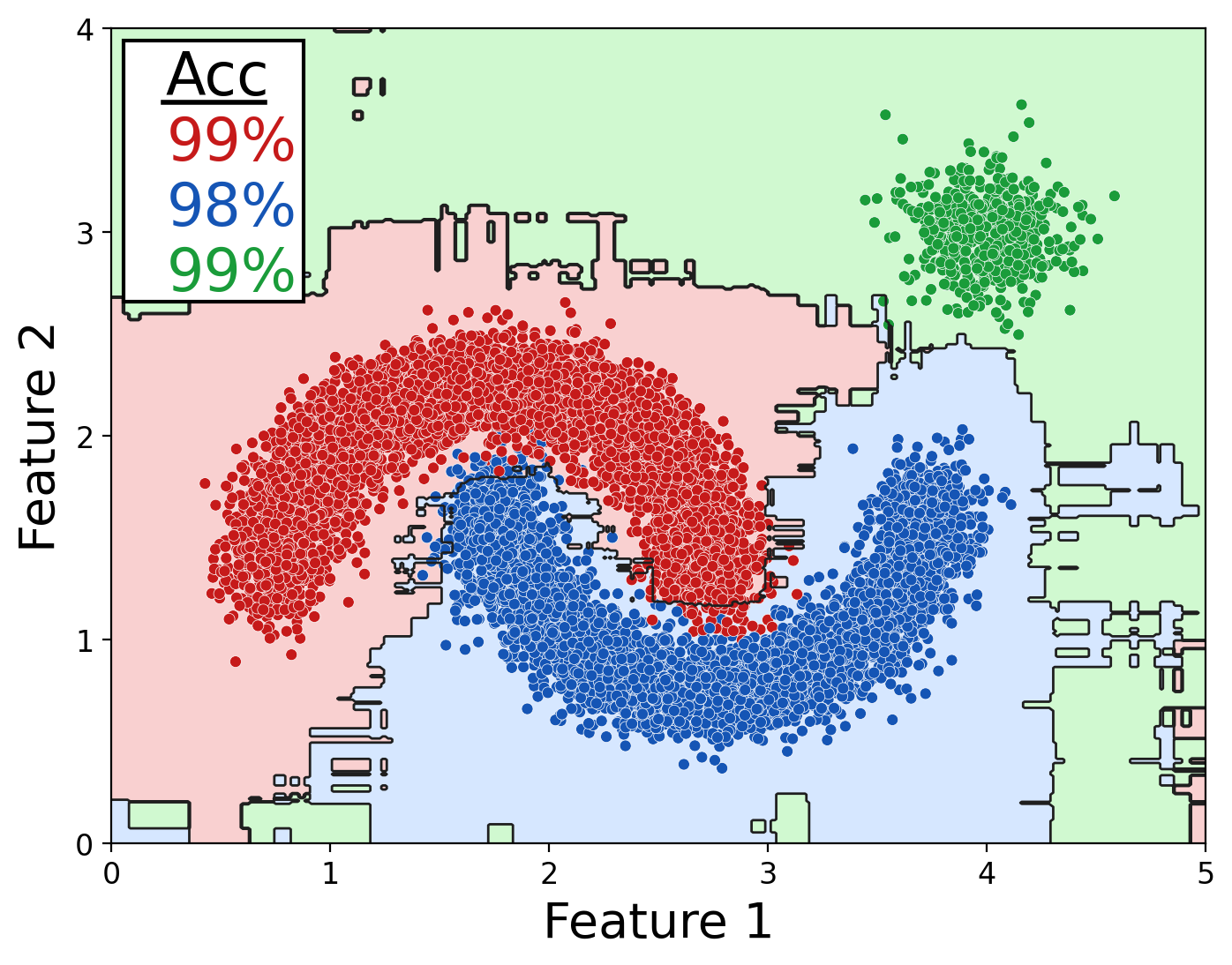} &
            \includegraphics[width=0.32\linewidth, valign=m]{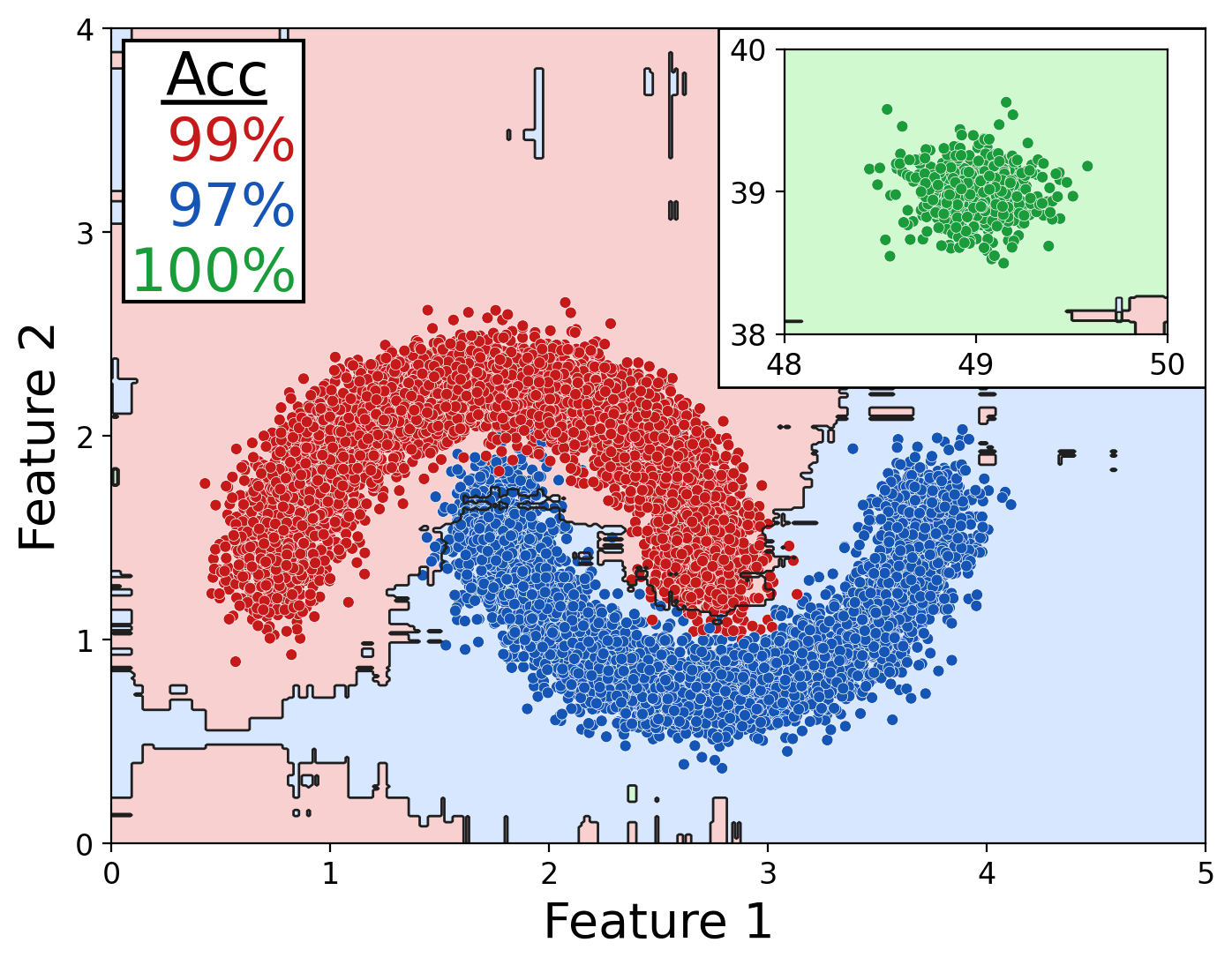} &
            \includegraphics[width=0.32\linewidth, valign=m]{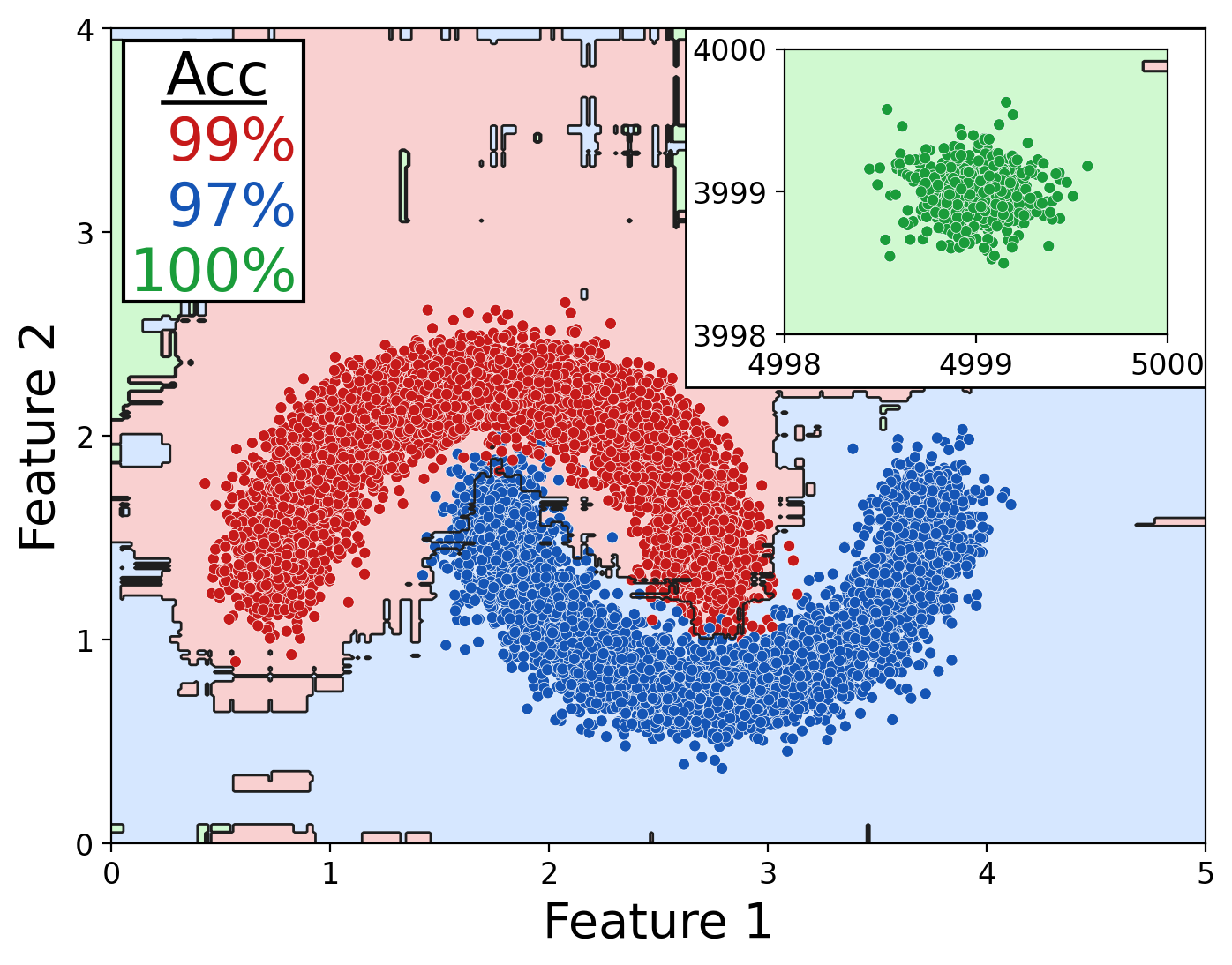} \\
        \end{tabular}
    \end{minipage}\hfill
    \begin{minipage}{0.28\textwidth}
        \vspace{0.1cm}
        \caption{Toy example. Top row: fully random trees (depth=10) under $\infty$-DP. Bottom row: \methodname{} under $(2,10^{-6})$-DP. As the outlier class boundary shifts and the feature-space range changes (left to right), fully random trees fail to capture the structure, whereas our method continues to construct accurate trees.\looseness=-1}
        \label{fig:moons-toy-example}
    \end{minipage}
\end{figure}
% Settings used for above experiment
% name: dp_random_forest
% n_estimators: 25
% epsilon: 2.0
% delta: 1e-6
% splitting_fraction: 0.75
% splitting_algorithm: sparse
% leaf_algorithm: exponential_mechanism
% max_depth: 100
% random_state: 84
% n_jobs: null
% verbose: 1

\paragraph{Contributions.} We introduce a novel differentially private random forest algorithm, \methodname{}. Our approach combines randomized tree construction, inspired by \citep{geurts2006extremely} and prior work, with a principled privacy-aware mechanism that determines \emph{when to stop splitting}. Concretely, we allocate part of the privacy budget to test whether a node contains a sufficient number of data points in feature space, allowing us to prune empty or low-density branches.

A key challenge is enforcing this stopping criterion under differential privacy. We cast the problem of identifying split nodes as a private heavy-hitters problem over a tree of height $h$.
While related problems have been studied by \cite{GhaziK0M023,BiswasCKSZ24}, their techniques do not directly apply. A naive approach that adds independent noise to each node count suffers from poor utility for deep trees, due to $\ell_2$ sensitivity (which determines the magnitude of noise needed for privacy) scaling as $\sqrt{h}$.
To overcome this, we exploit the monotone structure of rooted trees. Our key idea is to perform structured, highly overlapping noisy binary searches over the tree by querying nodes in a middle layer and propagating constraints upward and downward: heavy nodes imply heavy ancestors, while light nodes imply light descendants. As illustrated in \Cref{fig:algorithm-phases}, this induces consistency constraints that reduce the problem to smaller independent subproblems of size at most $h/2$, reducing sensitivity to $\sqrt{1 + \log_2 h}$. We further integrate techniques from sparse histogram estimation \citep{WilkinsKZK24}, improving both utility and computational efficiency.
% We want to support deep trees to improve applicability, since the random splitting process often need more splits than greedy random forests for strong performance.
% Our core technical insight leverages the monotonic structure of rooted trees to overcome this limitation. 
% We introduce a novel algorithm that conceptually performs highly overlapping noisy binary searches.
% We query each node in the middle layer of the tree. For any heavy node it's ancestors must also be heavy, and likewise the descendants of light nodes must be light. 
% \Cref{fig:algorithm-phases} shows how we mark implications from the middle layer. This allows us to reduce the problem to non-overlapping subproblems of size $\leq h/2$ which reduces the sensitivity to $\sqrt{1 + \log h}$.
% We combine this technique with methods from sparse histogram estimation~\citep{WilkinsKZK24} which significantly improves utility and computational performance.
Finally, we develop a unified privacy analysis that bounds the privacy loss of the entire pruning procedure jointly over the forest construction, rather than relying on standard composition across trees, yielding a tighter privacy accounting.
% Furthermore, we introduce a new approach to analyzing the privacy of the forest that operates jointly over the entire forest construction, directly bounding the privacy loss of the pruning mechanism rather than relying on composition theorems across individual trees.

Empirically, \methodname{} enables the training of deep, high-utility random forests under strict privacy constraints, achieving new state-of-the-art results and consistently outperforming existing baselines across multiple datasets and privacy regimes. %To facilitate immediate practical deployment and bridge the gap between theoretical DPML and software ecosystems, we provide a fully \texttt{scikit-learn} compatible implementation of our algorithm. 

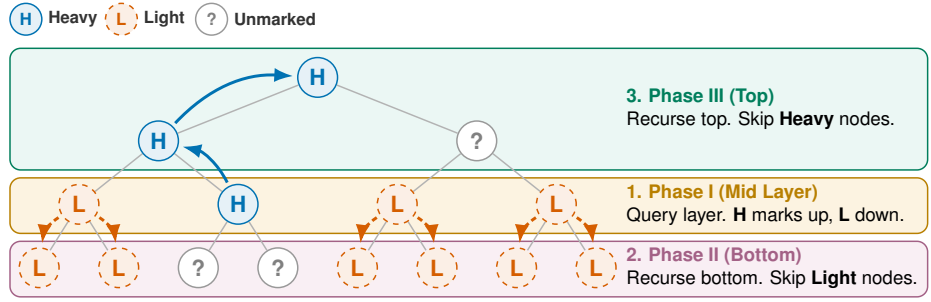
\begin{figure}[t]
  \centering
  \resizebox{\linewidth}{!}{%
      \usetikzlibrary{arrows.meta, backgrounds, calc, positioning}

% ==========================================
% Colorblind-Friendly Palette (Okabe-Ito)
% ==========================================
\definecolor{cbOrange}{RGB}{230, 159, 0}
\definecolor{cbSkyBlue}{RGB}{86, 180, 233}
\definecolor{cbGreen}{RGB}{0, 158, 115}
\definecolor{cbBlue}{RGB}{0, 114, 178}
\definecolor{cbVermilion}{RGB}{213, 94, 0}
\definecolor{cbPurple}{RGB}{204, 121, 167}

% NOTE: No empty lines allowed inside the [...] style block below!
\begin{tikzpicture}[
    basic/.style={circle, draw=black, thick, minimum size=7.5mm, inner sep=0pt, font=\sffamily\bfseries\large},
    heavy/.style={basic, draw=cbBlue, fill=cbBlue!10, text=cbBlue, thick},
    light/.style={basic, draw=cbVermilion, fill=cbVermilion!10, dashed, text=cbVermilion, thick},
    unmarked/.style={basic, draw=black!40, fill=white, text=black!50, thick},
    phase3box/.style={rounded corners=2mm, draw=cbGreen!80!black, thick, fill=cbGreen, fill opacity=0.08},
    phase1box/.style={rounded corners=2mm, draw=cbOrange!80!black, thick, fill=cbOrange, fill opacity=0.12},
    phase2box/.style={rounded corners=2mm, draw=cbPurple!80!black, thick, fill=cbPurple, fill opacity=0.12},
    propup/.style={ultra thick, cbBlue, -Latex, shorten >=2pt, shorten <=2pt},
    propdown/.style={ultra thick, cbVermilion, densely dashed, -Latex, shorten >=2pt, shorten <=2pt}
]

% ==========================================
% 1. TREE NODES (8 Leaves)
% ==========================================
% Level 0
\node[heavy] (L0) at (0, 0) {H};

% Level 1
\node[heavy] (L1A) at (-3, -1.2) {H};
\node[unmarked] (L1B) at (3, -1.2) {?};

% Level 2 (Mid Layer)
\node[light] (L2A) at (-4.5, -2.4) {L};
\node[heavy] (L2B) at (-1.5, -2.4) {H};
\node[light] (L2C) at (1.5, -2.4) {L};
\node[light] (L2D) at (4.5, -2.4) {L};

% Level 3 (Leaves)
\node[light] (L3A) at (-5.25, -3.6) {L};
\node[light] (L3B) at (-3.75, -3.6) {L};
\node[unmarked] (L3C) at (-2.25, -3.6) {?};
\node[unmarked] (L3D) at (-0.75, -3.6) {?};
\node[light] (L3E) at (0.75, -3.6) {L};
\node[light] (L3F) at (2.25, -3.6) {L};
\node[light] (L3G) at (3.75, -3.6) {L};
\node[light] (L3H) at (5.25, -3.6) {L};

% ==========================================
% 2. BACKGROUND BANDS & STANDARD EDGES
% ==========================================
\pgfdeclarelayer{background}
\pgfsetlayers{background,main}

\begin{scope}[on background layer]
    % Bands extended right to host the phase explanation labels
    \filldraw[phase3box] (-5.8, 0.55) rectangle (11.5, -1.75);
    \filldraw[phase1box] (-5.8, -1.9) rectangle (11.5, -2.95);
    \filldraw[phase2box] (-5.8, -3.1) rectangle (11.5, -4.15);

    % Standard Straight Edges (Behind nodes)
    \foreach \u/\v in {L0/L1A, L0/L1B, L1A/L2A, L1A/L2B, L1B/L2C, L1B/L2D,
                       L2A/L3A, L2A/L3B, L2B/L3C, L2B/L3D,
                       L2C/L3E, L2C/L3F, L2D/L3G, L2D/L3H} {
        \draw[thick, black!30] (\u) -- (\v);
    }
\end{scope}

% ==========================================
% 3. PROPAGATING ARROWS (Visualizing Logic)
% ==========================================
\draw[propup] (L2B) to[bend right=25] (L1A);
\draw[propup] (L1A) to[bend left=25] (L0);

\draw[propdown] (L2A) to[bend right=25] (L3A);
\draw[propdown] (L2A) to[bend left=25] (L3B);
\draw[propdown] (L2C) to[bend right=25] (L3E);
\draw[propdown] (L2C) to[bend left=25] (L3F);
\draw[propdown] (L2D) to[bend right=25] (L3G);
\draw[propdown] (L2D) to[bend left=25] (L3H);

% ==========================================
% 4. LABELS AND DESCRIPTIVE TEXT
% ==========================================
% Legend
\begin{scope}[shift={(-5.5, 1.1)}]
    \node[heavy, minimum size=6mm, font=\sffamily\bfseries\small] (legH) at (0, 0) {H};
    \node[anchor=west, font=\sffamily\small, text=black!80] at (0.3, 0) {\textbf{Heavy}};
    \node[light, minimum size=6mm, font=\sffamily\bfseries\small] (legL) at (1.8, 0) {L};
    \node[anchor=west, font=\sffamily\small, text=black!80] at (2.1, 0) {\textbf{Light}};
    \node[unmarked, minimum size=6mm, font=\sffamily\bfseries\small] (legU) at (3.5, 0) {?};
    \node[anchor=west, font=\sffamily\small, text=black!80] at (3.8, 0) {\textbf{Unmarked}};
\end{scope}

% Phase Explanations (Right Margin)
\node[anchor=west, font=\sffamily, align=left] at (5.7, -0.6) {
    \textbf{\color{cbGreen!80!black}3. Phase III (Top)}\\[0.5mm]
    Recurse top. Skip \textbf{Heavy} nodes.
};

\node[anchor=west, font=\sffamily, align=left] at (5.7, -2.4) {
    \textbf{\color{cbOrange!80!black}1. Phase I (Mid Layer)}\\[0.5mm]
    Query layer. \textbf{H} marks up, \textbf{L} down.
};

\node[anchor=west, font=\sffamily, align=left] at (5.7, -3.6) {
    \textbf{\color{cbPurple!80!black}2. Phase II (Bottom)}\\[0.5mm]
    Recurse bottom. Skip \textbf{Light} nodes.
};

\end{tikzpicture}\hspace*{1.2cm}%
  }
  \caption{%
    Example of a recursive iteration of our heavy hitter detector (\Cref{alg:heavy_hitters}). 
    \textbf{Phase~I} (orange band) queries all nodes at the middle layer and
    marks them \emph{Heavy} (\textbf{H}, blue) or \emph{Light}
    (\textbf{L}, dashed orange), with implications for ancestors and descendants shown by arrows.
    \textbf{Phase~II} (purple band) recursively explores subtrees containing unmarked nodes (\textbf{?}, gray), skipping \emph{Light} nodes.
    \textbf{Phase~III} (green band) recursively marks the upper half of the tree, skipping \emph{Heavy} nodes.
    % We never relabel nodes, so we can skip nodes marked \textbf{H} in this recursive call.
    %\cl{We could add dashed lines or shading to show the 3 subtrees that we will recursively explore - 2 nodes with ? in Phase II and the 3 nodes in the top in Phase III.}
    % Three-phase tree algorithm anchored at the mid layer $L_{\mathrm{mid}}$.
    % \textbf{Phase~I} (orange band) queries every node at $L_{\mathrm{mid}}$ and
    % marks it \emph{Heavy} (\textbf{H}, blue) or \emph{Light}
    % (\textbf{L}, dashed orange).
    % \textbf{Phase~II} (purple band) propagates Light marks down: every
    % descendant of a Light node at $L_{\mathrm{mid}}$ is itself Light
    % (dashed orange arrows).
    % \textbf{Phase~III} (green band) propagates Heavy marks up: every ancestor
    % of a Heavy node is itself Heavy (solid blue arrows).
  }
  \label{fig:algorithm-phases}
\end{figure}

\section{Preliminaries}
\label{sec:preliminaries}

%\ab{I think we need basic background on RFs and especially a clear explanation of how fully random trees are built (including estimation in leafs), so the reader can understand that our problem reduces to doing this construction privately without paying too much for the height, and thus understand the related work discussion in Sec 3 and the HH algorithm in Sec 4}

\textbf{Problem setup.} We consider a dataset $D \in \mathcal{X}^n$ of $n$ data points, where each point $x \in D$ has a set of features and a label $y \in \mathcal{Y}$. Features may be numerical or categorical; for numerical features we assume known bounds, and for categorical features a known domain.
Our goal is to privately train a random forest to predict labels for unseen data. Random forests support both classification and regression; our contribution is a new differentially private tree-building algorithm applicable to both settings. We focus on classification in our presentation, as it is the primary setting in prior work.

\textbf{Random forests.} A Random Forest (RF) is an ensemble method that combines predictions from multiple decision trees. Each tree recursively partitions the feature space based on feature-value splits until a stopping criterion is reached (e.g., maximum depth).
%Standard decision tree induction algorithms (such as CART \citep{breiman1984classification}) recursively partition the feature space by selecting splits that greedily maximize an information-theoretic criterion, such as Gini impurity or information gain. 
% For each leaf, the training points that fall into the partition are used to construct a score function $s: \mathcal{Y} \rightarrow [0,1]$ used for future prediction.
Each leaf corresponds to a region of the feature space and defines a score function $s: \mathcal{Y} \rightarrow [0,1]$, typically derived from the training samples in that leaf.
To predict, each tree outputs class scores, and the forest aggregates them (e.g., by summation) to select the class with highest total score. Common choices include majority vote, where each tree assigns all weight to a single class, or soft voting, where scores reflect class proportions within each leaf.\looseness=-1

\textbf{Differential privacy.}
Differential privacy~\citep{DworkMNS06} (DP) is a framework for rigorous, quantifiable privacy guarantees. 
Two datasets $D \sim D'$ are neighbors if they differ by a single record (i.e., one addition or removal).
Differential privacy requires that an algorithm's output distribution is approximately indistinguishable on any pair of neighboring datasets.
Several variants have been proposed to suit different settings.
% have come up with a variety of different formulations of differential privacy, adapted to different use cases. %(See \cite{desfontaines2019sok} for a (slightly outdated) overview). \cl{Should we remove this reference? It seems out-of-place.}
In this work, we report privacy guarantees in $(\varepsilon, \delta)$-DP and use zero-concentrated differential privacy (zCDP) for simpler privacy accounting.

\begin{definition}[{\citep{DworkR14}} $(\varepsilon, \delta)$-Differential Privacy]
    \label{def:differential-privacy}
A randomized mechanism 
$\mathcal{M}: \mathcal{X}^{*} \rightarrow \mathcal{Z}$ 
satisfies $(\epsilon, \delta)$-DP, if for any $D \sim D'$ and every measurable set of outputs $Z \in \mathcal{Z}$ we have
\[
    \Pr[\mathcal{M}(D) \in Z] \leq e^\epsilon \Pr[\mathcal{M}(D') \in Z] + \delta \,.
\]
\end{definition}
%\cl{I changed Y to Z because Y is typically used for the labels. This must be updated throughout the paper. In particular in the appendix.}

\begin{remark}
 When $\delta = 0$, the guarantee is referred to as pure DP, or $\varepsilon$-DP.
 % Prior work on differentially private random forests has largely focused on $\varepsilon$-DP.
 % In the related work section, we argue that $\varepsilon$-DP is poorly-suited for random forests and view this a fundamental limitation of this line of work. 
\end{remark}

\begin{definition}[\citep{BunSteinke16} zero-Concentrated Differential Privacy]
\label{def:zCDP}
    A randomized mechanism $\mathcal{M}: \mathcal{X}^{*} \rightarrow \mathcal{Z}$ satisfies $\rho$-zCDP, if for any $D \sim D'$ and all $\alpha \in (1, \infty)$, it holds that
    \[
    D_\alpha(\mathcal{M}(D) \parallel \mathcal{M}(D')) \le \rho \alpha\,,
    \]
    where $D_\alpha(\cdot \parallel \cdot)$ is the $\alpha$-Rényi divergence between two distributions.
\end{definition}

We present some standard useful lemmas for $(\varepsilon,\delta)$-Differential Privacy and $\rho$-zCDP in \Cref{app:additional-preliminaries}.

\section{Related Work}
\label{sec:related}

\textbf{Differentially private random forests.}
Existing DP-RF methods fall into two main paradigms.
% \label{sec:related-tree-building}

The first paradigm \citep{PatilS14_random_forest,FletcherI15_greedy,li2017random,xin2019differentially,HouLMNCL19,GuanSSWD20,zhang2021random,consul2021differentially,vos2023differentially,Suihkonen2023DP_RF} allocates the privacy budget between private tree construction and leaf queries. Trees are built using impurity-based criteria similar to standard decision trees (e.g. \cite{breiman2001random}), typically via the exponential mechanism \citep{McSherryT07}. We omit a detailed description due to the fact that our approach aligns more closely with the second paradigm, and because we identify privacy issues in all such works (see \cref{app:privacy-concerns}).

The second paradigm, loosely inspired by extremely randomized trees ("Extra Trees") \citep{geurts2006extremely}, constructs trees via random splitting \citep{JagannathanPW12_small_random,JagannathanMP13_semi_private_data,FletcherI15_random,fletcher2017differentially,holohan2019diffprivlibibmdifferentialprivacy}. In these methods, splits are chosen independently of the data. 
%While this avoids spending privacy budget on structure learning, it introduces key limitations: (i) random thresholds can produce empty or low-quality leaves; (ii) typically only a single random split is considered per node, increasing variance; and most critically, (iii) trees are fully expanded to a fixed depth without data-dependent stopping. 
%As a result, these approaches poorly adapt to the data and often leads to degraded performance, 
A key limitation of these techniques is that trees are fully expanded to a fixed depth without data-dependent stopping. 
While this avoids spending privacy budget on structure learning, these approaches poorly adapt to the data and often leads to degraded performance,
as illustrated in \Cref{fig:moons-toy-example} and in our experiments of Section~\ref{sec:experiments}.\looseness=-1

Our work departs from both paradigms by introducing a simple intermediate idea: we retain randomized splitting but use part of the privacy budget to determine when to stop splitting. Concretely, we build deep random trees and then privately prune them by identifying nodes that are sufficiently ``heavy'' in feature space; nodes whose parent is not heavy are removed. This allows data-dependent depth control without fully reverting to greedy tree construction.

Leaf prediction is handled either by adding noise to label counts (e.g. \cite{FletcherI15_greedy}) or by private majority selection (e.g. \citep{fletcher2017differentially}). We support both variants in our implementation (see \Cref{app:implementation}).

\begin{remark}
Most of the existing DP-RF methods rely on pure $\varepsilon$-DP, which is generally suboptimal for ensemble learning where privacy must be composed across many tree operations (e.g., repeated splits and leaf queries across many trees). Indeed, privacy costs compose linearly under $\varepsilon$-DP, leading to significantly higher noise compared to the approximate or concentrated DP variants we use in this work.\looseness=-1
\end{remark}

\textbf{Differentially private heavy hitters in trees.}
% \label{sec:related-heavy-hitters}
% A key technical component of our work is a novel algorithm for heavy hitters detection in hierarchical data.
% For this problem, we are given a tree $\T$ with height $h$. Each data point increase the count of a single leaf by 1, and each internal node stores the sum of counts for all child nodes.
% We want to privately detect all nodes with count above some threshold $\tau$. Notice that the root of $\T$ is always heavy and the set of heavy hitters forms a subtree of $\T$ (except for the special case where no nodes are heavy). We exploit this structure in our algorithm.
% This problem has connections to several well-studied problems in differential privacy, such as quantile estimation and private release of sparse histograms.
A key technical component of our work is a novel algorithm for heavy-hitters detection in hierarchical data. This problem is related to several classical tasks in differential privacy, including quantile estimation \citep{kaplan2022differentially,HuangLY21,BunNSV15,Cohen0NSS23} and private release of sparse histograms \citep{Korolova09,CormodePST12,BalcerV19,WilsonZLDSG20,ALP22,LebedaR25,KerschbaumLeeWu25}.
Another related problem is count estimation in trees, where noise scales with $O(\sqrt{h})$ with $h$ the height of the tree.
Despite extensive work on heavy hitters and hierarchical data release, there is limited work at their intersection. \cite{GhaziK0M023} study a closely related problem (see Problems 4.2 and 4.4), achieving a reduced dependence of $O(\sqrt{\log h})$ at the cost of an additional $\sqrt{n}$ term, which is typically much larger than $\sqrt{h}$. Another relevant prior work is \cite{BiswasCKSZ24}, who consider private release of Hierarchical Heavy Hitters (HHH), a strict generalization of our setting; indeed, our problem can be addressed by post-processing the output for the HHH problem. While their construction would be applicable to our setting, we identify a subtle error in their privacy analysis (see \Cref{app:HHH-bug}).
%Due to space limitations, we provide an expanded discussion in \Cref{app:related-heavy}.
%\update{
The authors recently uploaded a revised version of their paper to arXiv, correcting the mistake in the original version~\citep{biswas2026}. Their corrected approach has worse asymptotic dependence than ours in terms of both $h$ and $\delta$, as well as larger constants. Consequently, our approach achieves substantially better performance in practice (see Table~\ref{tab:comparison}).
%the same desirable dependence on $h$, but has worse error in terms of the privacy parameters.
%}
We provide an expanded discussion of DP heavy hitters in \Cref{app:related-heavy}.

\section{A Private Heavy Hitters Algorithm for Hierarchical Data}
\label{sec:pruning-algorithms}

In this section, we present our novel private heavy hitters algorithm that will form the basis of our private random forest algorithm. 
We are given a tree structure $\mathcal{T}$ of height\footnote{We define the height as the number of nodes on a root-to-leaf path, including both endpoints. An alternative convention counts edges instead and thus differs by 1. We use depth to refer to this alternative convention. Thus, $\text{height}(\mathcal{T}) = \text{depth}(\mathcal{T}) + 1$.} $h$. Each data point contributes to the count of a leaf node in $\mathcal{T}$, and internal node counts are defined as the sum of their children’s counts. Our goal is to privately identify all nodes whose count exceeds a threshold $\tau$, while ensuring that the selected nodes form a valid subtree.

% A baseline technique to keep in mind traverses the tree in top-down order. At each node, we privately compare the count with the desired threshold. 
% We keep adding nodes above the threshold to our tree and whenever a node is below the threshold we stop exploring that part of the tree.
% It is easy to see that the counts between neighboring datasets only differ on a single path down the tree. We can use the Gaussian mechanism to perform the threshold queries under zCDP by adding noise to all counts with scale $\sigma=\sqrt{h/2\rho}$.
% Note that since prior work relied on $\varepsilon$-DP the analogue baseline adds noise from the Laplace distribution with scale $h/\varepsilon$. 

\textbf{Baseline.} A simple baseline traverses the tree in a top-down manner. At each node, we privately test whether its count exceeds the threshold $\tau$, retaining nodes that pass and pruning subtrees whenever a node falls below the threshold. Since neighboring datasets affect counts along a single root-to-leaf path, the sensitivity is $\sqrt{h}$. We can thus apply the Gaussian mechanism to each count with noise scale $\sigma = \sqrt{h/(2\rho)}$ to satisfy $\rho$-zCDP. Our algorithm significantly improve upon this scaling in $h$.

\subsection{Binary Search-Inspired Heavy Hitters Detection}

\textbf{High-level presentation of the approach.} The pseudocode of our algorithm for finding heavy hitters in hierarchical data is given in \Cref{alg:heavy_hitters}. The algorithm is parameterized by a private thresholding primitive $\textsc{CheckThreshold}$, which privately tests whether a node count exceeds a threshold $\tau$ (it is "heavy"). The overall error of our method will be determined by the error of this primitive.

Conceptually, the algorithm performs $2^h$ highly overlapping noisy binary searches over the tree. To control privacy cost, nodes already identified as heavy by $\textsc{CheckThreshold}$ in one search are never queried again by another overlapping search. \Cref{fig:algorithm-phases} illustrates one iteration of the procedure. While noisy binary search is a classical tool for private quantile estimation in one dimension (e.g. the PrivQuant algorithm \citep{HuangLY21}), our main contribution is extending this idea to hierarchical tree structures and integrating it with sparse histogram techniques.

Our key observation is that each data point influences at most $1 + \floor{\log_2 h}$ threshold queries, yielding an $\ell_2$-sensitivity of $\sqrt{1+\floor{\log_2 h}}$. This is an exponential improvement over the $\sqrt{h}$ sensitivity of the standard top-down baseline. A naive implementation would still be impractical for deep trees due to the enormous number of empty nodes. To address this, we implement $\textsc{CheckThreshold}$ using techniques from the sparse histogram literature (Algorithm~\ref{alg:sparse_adaptive_queries}), allowing the algorithm to efficiently query only relevant (non-empty) nodes. This substantially improves both utility and computational efficiency. Finally, we derive tight $(\varepsilon,\delta)$-DP guarantees by analyzing the privacy loss of the pruning procedure jointly, rather than through standard composition bounds.

\begin{algorithm}[h]
\caption{\textsc{MarkHeavyHitters}($\T$, $h$)}
\label{alg:heavy_hitters}
\begin{algorithmic}[1]
\REQUIRE A tree $\T$ with height $h$ and a threshold $\tau \in \mathbb{R}$. 
%\REQUIRE Parameters $\sigma$, $\tau$ and $\Delta$.
\REQUIRE A DP algorithm $\textsc{CheckThreshold}_\tau: \mathbb{N} \rightarrow \{\top, \bot\}$.

\COMMENT{Phase 1: Query the middle layer}
\STATE Let $L_{\text{mid}}$ be the set of nodes at level $\lfloor h/2 \rfloor$ of $\T$. 
\FORALL{node $u \in L_{\text{mid}}$ where $u$ has not been marked} \label{line:condition-unmarked}
    \IF{$\textsc{CheckThreshold}_\tau(\#count(u))$ = $\top$}
        \STATE Mark $u$ and all ancestors of $u$ as \textit{Heavy}.
    \ELSE
        \STATE Mark $u$ and all descendants of $u$ as \textit{Light}.
    \ENDIF
\ENDFOR

\COMMENT{Phase 2: Recursively apply the algorithm to children that were not marked \textit{Light} in Phase 1} 
%\STATE\cl{TODO: Needs a slightly different phrasing/notation for this loop. This has an off-by-one out-out-bounds issue when h=1 or h=2. It should only be run for h>2}
\IF{$h \geq 3$}
    \STATE Let $L_{\text{mid + 1}}$ be the set of nodes at level $\lfloor h/2 \rfloor + 1$ of $\T$.
    \FORALL{node $v \in L_{\text{mid} + 1}$ where $v$ is NOT marked \textit{Light}} 
        %\IF{$v$ is NOT marked \textit{Light}} 
        \label{line:condition-recursive-nonlight}
            \STATE Let $\T_v$ be the subtree of $\T$ rooted at $v$.
            \STATE \textsc{MarkHeavyHitters}($\T_v$, $h - \lfloor h/2 \rfloor - 1$).
        %\ENDIF
    \ENDFOR
\ENDIF

\COMMENT{Phase 3: Recursive apply the algorithm in the top part of the tree}
\IF{$h \geq 2$}
    \STATE Let $\hat{\T}$ be $\T$ restricted to layers $\{0,\dots,\text{mid} - 1\}$. 
    \STATE \textsc{MarkHeavyHitters}($\hat{\T}$, $\lfloor h/2 \rfloor$).
\ENDIF

%\RETURN The subtree of all nodes marked \textit{Heavy}.
\end{algorithmic}
\end{algorithm}

\textbf{Technical details and analysis.} We begin by establishing some basic properties of \Cref{alg:heavy_hitters}. Formal proofs are deferred to \Cref{app:proofs}; here we focus on the intuition behind the main results.

\begin{lemma} \label{lem:bin_search_sensitivity}
    For any tree $\T$ with height $h$ and any $x \in \mathcal{X}$, \Cref{alg:heavy_hitters} calls $\textsc{CheckThreshold}$ with at most $1 + \floor{\log_2(h)}$ (equivalently $\ceil{\log_2(h + 1)})$ nodes that contain $x$.
\end{lemma}

\begin{lemma} \label{lem:error_generic}
    The error of \Cref{alg:heavy_hitters} is bounded by the error of $\textsc{CheckThreshold}$. Specifically, if all calls to $\textsc{CheckThreshold}$ where $\#count(u) \leq \tau - \alpha$ returns $\bot$ then all nodes in $\T$ with count $\leq \tau - \alpha$ are labeled \textit{Light}. Similarly, if $\textsc{CheckThreshold}$ returns $\top$ for all calls where $\#count(u) \geq \tau + \beta$ then all nodes with count $\geq \tau + \beta$ are labeled \textit{Heavy}.
\end{lemma}

% \begin{proof}
%     It follows from the monotonicity of the tree structure. The statement clearly holds for any nodes that are queried by $\textsc{CheckThreshold}$. 
%     We have to show that it also holds for nodes that are marked without being directly queried.
%     Nodes that are marked \textit{Light} are descendants of some node $u$ where we had $\textsc{CheckThreshold}_{\tau}(\#count(u)) = \bot$. These nodes therefore have a count $< \tau + \beta$ since $u$ cannot have a lower count than it's descendant. Similarly, nodes that are marked \textit{Heavy} are ancestors of some node where $\textsc{CheckThreshold}$ returned $\top$ and thus must have a count of $> \tau - \alpha$.
% \end{proof}

For \Cref{lem:bin_search_sensitivity}, observe that each data point affects only a single query in Phase~I and at most one recursive call of height at most $h/2$ (see \Cref{fig:algorithm-phases} for an illustration). This recursive structure directly yields the logarithmic sensitivity bound.
For \Cref{lem:error_generic}, we exploit the monotonicity of the tree. If a node is marked \textit{Heavy} due to a heavy descendant, then the node itself must indeed be heavy since counts can only increase toward the root. Conversely, descendants of a \textit{Light} node must also be light. These implications allow labels to propagate through the tree without additional queries.

% By \Cref{lem:bin_search_sensitivity}, \Cref{alg:heavy_hitters} satisfies $\rho$-zCDP if we instantiate $\textsc{CheckThreshold}$ with the Gaussian mechanism, similar to the baseline but now with $\ell_2$-sensitivity just $\sqrt{1 + \floor{\log_2 h}}$ (see \Cref{lem:gaussian-mech}).
% However, a straight-forward implementation is infeasible for deep trees since we would need to query $2^{\floor{h/2}}$ nodes in Phase 1. 
% This not only lead to extreme computational requirements in term of time and memory but also error proportional to $O(\sigma \sqrt{\ln{2^{\floor{h/2}}}}) = O_{\varepsilon,\delta}(\sqrt{h \log(h)})$. %\de{it is a bit unclear to me why rho appears here, we are not using zCDP in the splitting.}
% Most of these nodes will be empty if $h/2 > \log_2(n)$, so we adapt techniques from sparse histograms.
By \Cref{lem:bin_search_sensitivity}, \Cref{alg:heavy_hitters} satisfies $\rho$-zCDP if we instantiate $\textsc{CheckThreshold}$ with the Gaussian mechanism, analogous to the baseline but with $\ell_2$-sensitivity reduced to $\sqrt{1+\floor{\log_2 h}}$ (see \Cref{lem:gaussian-mech}). However, a direct implementation remains impractical for deep trees, since Phase~I alone queries $2^{\floor{h/2}}$ nodes. This leads not only to prohibitive computational and memory costs, but also to error scaling as $O(\sigma \sqrt{\ln(2^{\floor{h/2}})}) = O_{\varepsilon,\delta}(\sqrt{h\log h})$. Since most of these nodes are empty whenever $h/2 > \log_2 n$, we instead adapt techniques from sparse histogram estimation.
Specifically, we introduce an adaptive variant of the method from \cite{WilkinsKZK24} to implement $\textsc{CheckThreshold}$ (see Algorithm~\ref{alg:sparse_adaptive_queries}). Threshold queries with sufficiently small counts are deterministically mapped to $\bot$ (capturing all empty nodes under suitable parameters), while Gaussian noise is added only for the remaining queries before comparing against the threshold.\looseness=-1
% First, we present an adaptive variant of the algorithm from \cite{WilkinsKZK24}.
% For each threshold query, we deterministically output $\bot$ for sufficiently small values (which includes all empty nodes for appropriate parameters). 
% For all other queries, we add Gaussian noise and compare the noisy count with the threshold.

\begin{algorithm}[h]
    \caption{GaussianSparseThreshold}\label{alg:sparse_adaptive_queries}
    \begin{algorithmic}[1]
        \REQUIRE Parameters $\sigma$, $\tau$, and $\Delta$. 
        \REQUIRE Dataset $D$, adaptive counting queries $q_1, q_2, q_3, \dots$
        \FORALL{$i \in \{1,2,3,\dots\}$}
        \STATE Sample $Z_i \sim \N\left(0, \sigma^2\right)$.
        \IF{$q_i(D) > \tau - \Delta - 1$ and $q_i(D) + Z_i > \tau$} 
        \STATE \textbf{Release} $\top$. 
        \ELSE
        \STATE \textbf{Release} $\bot$.
        \ENDIF
        \ENDFOR
    \end{algorithmic}
\end{algorithm}

\placeholderchristian{We could shift $\Delta$ by 1 and update the condition in the lemma and theory to include the offset. This makes the algorihtm and discussion cleaner since we require $\tau > \Delta$ instead of $\tau > \Delta + 1$. It requires a careful pass of the full paper.}

We present parameter settings for \Cref{alg:sparse_adaptive_queries} that ensure $(\varepsilon,\delta)$-DP in \Cref{app:sparse_gaussian}, and summarize them as part of our main theoretical result below. We next analyze the accuracy of the mechanism, measuring error in terms of worst-case misclassification. The following lemma follows from standard Gaussian tail bounds.

\begin{lemma}
    \label{lem:algorithm-error}
    Let $\alpha \in \mathbb{R}$ be the smallest value such that all nodes with count at least $\tau + \alpha$ are marked \textit{Heavy} by \Cref{alg:heavy_hitters}, and all nodes with count at most $\tau - \alpha$ are marked \textit{Light}. If $\textsc{CheckThreshold}$ is implemented via \Cref{alg:sparse_adaptive_queries} with parameters $\sigma > 0$ and $\tau \geq 1 + \Delta$, then with probability at least $1 - \beta$, we have
    \[
        \alpha \leq \sigma \cdot \sqrt{2 \ln (2 n\ceil{\log_2(h+1)}/\beta)} \enspace .
    \]
\end{lemma}

% \begin{proof}
%     Since $0 \leq \tau - \Delta - 1$ %$0 < \tau - \Delta$
%     all empty nodes are correctly classified as $\textit{Light}$. 
%     Let $m$ denote the number of calls to $\textsc{CheckThreshold}$ with non-zero counts performed by \Cref{alg:heavy_hitters}. 
%     It follows from \cref{lem:bin_search_sensitivity} that $m \leq n(1 + \floor{\log_2(h)}) = n\ceil{\log_2(h+1)}$. 
%     Let $(Z_1,\dots,Z_m)$ denote the noise samples for the non-zero counts.
%     By a standard Gaussian tail bound we have $\Pr[\vert Z_i \vert \geq t \cdot \sigma] \leq 2 \exp(- t^2/2)$ and thus $\Pr[\vert Z_i \vert \geq \alpha] \leq \beta / m$.
%     Using a union bound we have $\Pr[\max_{i \in [m]}\vert Z_i \vert \geq \alpha] \leq \beta$.
% \end{proof}

% We are now ready to present our main theoretical result. 
% We note that it is natural to allocate privacy budget for each tree and apply a composition theorem to obtain privacy parameters for the full forest. This is the approach of all prior work. Instead we analyze the privacy guarantee for the entire forest directly to avoid loose composition of $(\varepsilon,\delta)$-DP. 
We are now ready to present our main theoretical result. In contrast to prior work, which applies a composition theorem across trees to obtain privacy guarantees for the full forest, we directly analyze the privacy of the entire forest construction. This avoids the loss introduced by repeated composition under $(\varepsilon,\delta)$-DP and yields a tighter overall bound.

\begin{theorem}
    \label{thm:build-forest}
    Let $\mathcal{A}$ denote the algorithm that takes as input $k$ trees with maximum height $h$, and runs \Cref{alg:heavy_hitters} where $\textsc{CheckThreshold}$ is instantiated as \Cref{alg:sparse_adaptive_queries} with parameters $\sigma$, $\tau$, and $\Delta$. %All nodes that are not heavy hitters or children of heavy hitters are pruned. 
    Then $\mathcal{A}$ satisfies $(\varepsilon, \delta)$-DP where $m = k(1 + \floor{\log_2(h)})$, $\gamma(j) = (m - j)\log \Phi \left(\frac{\Delta}{\sigma}\right)$, and
    \begin{align*}
        &\delta \geq \max \bigg[ 1 - \Phi \left( \frac{\Delta}{\sigma} \right)^{m}, \\
        &\max_{j \in [m]} 1 - \Phi \left( \frac{\Delta}{\sigma} \right)^{m - j} + \Phi \left( \frac{\Delta}{\sigma} \right)^{m - j} \left[ \Phi\left(\frac{\sqrt{j}}{2\sigma} - \frac{(\varepsilon - \gamma(j))\sigma}{\sqrt{j}}\right) - e^{\varepsilon - \gamma(j)} \Phi\left(-\frac{\sqrt{j}}{2\sigma} - \frac{\left(\varepsilon - \gamma(j)\right)\sigma}{\sqrt{j}}\right) \right], \\
        &\max_{j \in [m]} \Phi\left(\frac{\sqrt{j}}{2\sigma} - \frac{\left(\varepsilon + \gamma(j)\right)\sigma}{\sqrt{j}}\right) - e^{\varepsilon + \gamma(j)} \Phi\left(-\frac{\sqrt{j}}{2\sigma} - \frac{(\varepsilon + \gamma(j))\sigma}{\sqrt{j}}\right) \bigg]\,.
    \end{align*}
    If $\tau \geq 1 + \Delta$, then with probability at least $1 - \beta$, the algorithm correctly classifies all nodes with count $\leq \tau - \alpha$ as \textit{Light} and all nodes with count $\geq \tau + \alpha$ as \textit{Heavy} where
    \[
        \alpha \leq \sigma \cdot \sqrt{2 \ln (2 kn\ceil{\log_2(h+1)}/\beta)} \enspace .
    \]
    Furthermore, 
    a variant of the algorithm that returns all nodes marked \textit{Heavy} can be implemented in time and space $O(knh)$.
    %assume that we can access the non-private counts of each node by traversing the tree from the root. 
    %Then the algorithm that returns all nodes marked \textit{Heavy} can be implemented in time and space $O(knh)$.
\end{theorem}

\begin{proof}
    By \Cref{lem:bin_search_sensitivity} adding or removing a data point affects at most $1 + \floor{\log(h)}$ queries for each tree. The privacy guarantees thus follow from \Cref{lem:privacy-adaptive} (see \Cref{app:sparse_gaussian}). 
    The error bound holds for each tree with probability at least $1 - \beta/k$ by \Cref{lem:algorithm-error}, and thus with probability at least $1 - \beta$ it holds for the entire forest simultaneously.
    For any $\tau \geq 1 + \Delta$, \Cref{alg:sparse_adaptive_queries} returns $\bot$ for all empty nodes. We may therefore avoid inspecting subtrees of empty nodes as we know these nodes would all be marked as $\textit{Light}$. This does not affect the output distribution and thus preserves the DP guarantee. 
    It is easy to see that there are at most $O(nh)$ non-empty nodes in each tree. We discuss implementation details to achieve space and running time linear in the number of non-empty nodes in \Cref{app:implementation}. \looseness=-1
\end{proof}

We use the parameters presented in \Cref{thm:build-forest} in our implementation, although the resulting privacy expressions are not very intuitive. For intuition, we provide below a simpler but looser closed-form bound in the regime $\varepsilon < 1$, which makes explicit the dependence on $h$, $k$, $\sigma$, and $\Delta$.
\begin{lemma}
    \label{lem:theorem-simpler}
    Define $\mathcal{A}$ as in \Cref{thm:build-forest}. If $\varepsilon < 1$, the parameters 
    \[
        \sigma = \frac{\sqrt{2m\ln(2.5/\delta)}}{\varepsilon} \quad \quad \quad
        \Delta = \sigma \cdot \sqrt{2\ln(2 m/\delta)} 
    \]
    satisfy the condition in \Cref{thm:build-forest} for $(\varepsilon, \delta)$-DP where $m = k(1 + \floor{\log_2 h})$.
\end{lemma}

% \begin{proof}
%     Here we use the looser add-the-deltas approach discussed in \cite{WilkinsKZK24}. The Gaussian noise is calibrated to $(\varepsilon, \delta/2)$-DP for $\ell_2$-sensitivity $\sqrt{m} = \sqrt{k(1 + \floor{\log_2 h})}$ using~\citep[Appendix~A]{DworkR14}. 
%     We have that $\Pr[\mathcal N(0, \sigma^2) \geq \Delta] \leq \delta/(2m)$ which implies that the maximum noise for $m$ queries is less than $\Delta$ with probability at least $1 - \delta/2$. 
% \end{proof}

\begin{remark}
    We give high probability bounds for the maximum error in \Cref{lem:algorithm-error} and \Cref{thm:build-forest}. However, notice that the error for nodes with small counts is deterministically bounded by $\Delta$. 
    We can achieve a deterministic error bound using a variant of \Cref{alg:sparse_adaptive_queries} as discussed in \Cref{app:double_threshold}.
\end{remark}

% \begin{remark}
%     \tc{Might be trivial, but it could be worth stating explicitly that what you rely on is the monotonicity of your hierarchical structure, i.e. $\texttt{count(parent)} \geq \texttt{count(child)}$. Essentially this approach shouldn't break if the tree structure ends up satisfying this (i am not sure it would make sense tho!} \de{That maybe relates to Aurélien's comments on properly defining the problem of defining heavy hitters. Then we can list the assumptions we actually need more generally and then argue that decision trees fulfill this. }
% \end{remark}
%\cl{Commented out for now. I'm not sure how to make this observation clear. I mentioned it in the related work section.}

\paragraph{Comparison with prior work.} Our technique has better asymptotic scaling than prior techniques \citep{GhaziK0M023,biswas2026} and, importantly, also much smaller constants. As a result, we outperform both approaches by several orders of magnitude across most practical parameter regimes. To illustrate this gap, we compare the minimum value of $\tau$ required by each technique to ignore empty nodes when finding heavy hitters ($\tau = 1 + \Delta$). We consider a single tree, fix $\delta = 10^{-6}$ and $h = 100$, and vary $\varepsilon$. Both \citep{GhaziK0M023} and \citep{biswas2026} require a failure probability as input; we conservatively set it to $1/4$ for both methods (smaller values would further increase their required $\tau$). The error of \citep{GhaziK0M023} also scales with the dataset size. We set $n=100\,000$, which is comparable to the dataset sizes used in our experiments.
The results in \Cref{tab:comparison} show the clear superiority of our approach.

\begin{center}
\label{tab:comparison}
\begin{tabular}{lrrrr}
\toprule
$\varepsilon$ & 0.1 & 0.5 & 1 & 4 \\
\midrule
Lumberjack (\Cref{thm:build-forest})      & 494             & 110           & 58          & 17          \\
\cite{GhaziK0M023}  & 25\,332         & 11\,328       & 8\,011      & 4\,005      \\
\cite{biswas2026} & 10\,402\,352    & 1\,654\,660   & 735\,637    & 138\,062    \\
\bottomrule
\end{tabular}
\captionof{table}{Smallest valid $\tau$ for a single tree ($\delta = 10^{-6}$).}
\end{center}

\subsection{Lower Bound for Deep Trees}
\label{sec:lower-bound}

By \Cref{thm:build-forest}, we should set the threshold $\tau$ to at least $1 + \Delta$.
This ensures an efficient implementation for deep trees and improves utility by allowing us to deterministically label empty nodes.
As a result, the threshold may be larger than what would be chosen in the non-private setting (e.g., in random forests), but we show below that this is unavoidable for any private method that supports deep trees. 

We now outline structural properties of heavy hitters algorithms that lead to our lower bound. 
First, observe that \Cref{alg:heavy_hitters} outputs at most $n$ nodes per layer when $\tau \geq 1 + \Delta$. 
% This is desirable, but we might accept an algorithm with slightly higher memory usage if it leads to better performance, so we expand the class of algorithms to those that always output at most $n^2$ nodes for each layer.
% Any algorithm exceeding this bound quickly becomes impractical for medium to large datasets, but the lower bound still applies for e.g. $n^3$ nodes with slightly different constants.
% We should also output all sufficiently heavy hitters with high probability. 
% These two desiderata lead to the following lower bound. \looseness=-1
While this is desirable, we allow for a broader class of algorithms that may output up to $n^2$ nodes per layer, since slightly larger memory usage can still be practical. Any method exceeding this regime quickly becomes infeasible for moderate or large datasets, though the lower bound still applies for e.g. $n^3$ nodes with slightly different constants. We additionally require that all sufficiently heavy nodes are recovered with high probability. Together, these conditions yield the following lower bound.

\begin{lemma}[Informal]
    \label{lem:informal-lower-bound}
    Let $\mathcal{M}$ denote an $(\varepsilon, \delta)$-DP mechanism (for any $\varepsilon > \delta^2$) that takes as input a tree $\T$ of height $h > 4 \log(n)$ and a dataset of size $n$, and outputs a set of at most $n^2$ nodes per layer of $\T$. Suppose that, with high probability, $\mathcal{M}$ outputs all nodes in $\T$ with count above some value $\hat\tau \in \mathbb{R}$. 
    Then $\hat\tau \geq \Omega\left(\frac{\min(h, \log(1/\delta))}{\varepsilon}\right)$.
\end{lemma}

The lemma follows from a lower bound for sparse histograms \cite[Theorem~7.2]{BalcerV19}, which (informally) states that any $(\varepsilon,\delta)$-DP mechanism that outputs few nonzero bins must incur error $\Omega\left(\frac{\min(\log d, \log(1/\delta))}{\varepsilon}\right)$ for $d$-dimensional histograms. We can encode such a histogram over any $d \leq 2^{h-1}$ in the leaf level of the tree. Technical details are deferred to \Cref{app:lower_bound}.

Note that our algorithm is at most a factor $O(\sqrt{\log h})$ from optimal for $\varepsilon < 1$ (see \cref{cor:deterministic-simpler,cor:deterministic-tree-error}). We emphasize that tight lower bounds are not the main focus of this work, and leave closing this gap as an open problem.

% Notice that for $\varepsilon < 1$, our algorithm is at most a factor $O(\sqrt{\log h})$ from optimal~(see \cref{cor:deterministic-simpler,cor:deterministic-tree-error}).
% We note that lower bounds are not a main focus of our work. 
% %In the proof we do not utilize any of the hierarchical structure between queries. 
% %\de{it wouldn't make sense, but using hierarchical structure in the proof sounds like the lower bound is actually too harsh and the gap should be bigger. Maybe there is a way to make this more clear/intuitive. The lower bound technically considers only one the "bottom layer", right? So for me, intuitively, it was like we query log h layers and that is where the factor probably comes from, but we don't want to constrain other algorithms to query in layers. Not sure if this makes sense.} \cl{It does. I'll try to be a bit more precise}
% We leave it as an open problem to close the gap between the upper and lower bound.
% %\cl{To David: Maybe we just remove the sentence like this? I'm not sure if we should claim that we believe the lower bound can be stronger. I go back and forth on that thought.}

\section{\methodname{}: Our Differentially Private Random Forest Method}
\label{sec:algorithm}

In this section, we describe how we use the heavy hitters algorithm from the previous section to instantiate our random forest method, \methodname{}.
We split the algorithm into two components: one for privately constructing the tree structures and one for privatizing the leaf predictions.
We ensure that constructing all trees together satisfies $(\varepsilon_1, \delta_1)$-DP and that privatizing all leaf nodes across all trees satisfies $(\varepsilon_2, \delta_2)$-DP. By basic composition, \methodname{} therefore satisfies $(\varepsilon_1 + \varepsilon_2, \delta_1 + \delta_2)$-DP.

\textbf{Tree construction.}
Given a maximum depth $d$, we recursively partition the feature space until this depth is reached.
We then pass the resulting tree, restricted to nodes with depth $\{0,\dots,d-1\}$ (corresponding to height $h=d$) to our heavy hitters algorithm.
The heavy hitters returned by the algorithm are retained as split nodes in the final tree. Any node whose parent is labeled $\textit{Light}$ is pruned, and the remaining $\textit{Light}$ nodes become the leaf nodes.\looseness=-1

To split the feature space, we select a feature uniformly at random.
For numerical features, we assume access to a range $[l,u]$ approximating the support of that feature (this is a standard assumption in DP).
We then sample a threshold uniformly at random, $t \in [l,u]$, and split the node accordingly.
The ranges for the left and right child nodes are updated to $[l,t]$ and $[t,u]$, respectively.
Categorical features are encoded using one-hot encoding, following common practice in non-private implementations (e.g., \texttt{sklearn} \citep{pedregosa2011scikit}). 
When splitting on a categorical feature, we select a category that has not already been used in an ancestor node.
Importantly, the privacy guarantees of \methodname{} are independent of the specific strategy used to generate random splits.
Exploring alternative splitting schemes together with (privacy-preserving) data preprocessing is an interesting direction for improving utility in future work.

% Note that many random forest algorithm sample data to train each tree using bootstrap or sampling without replacement \citep{pedregosa2011scikit, athey2019generalized}.
% These techniques are useful in avoiding overfitting and increasing the diversity of trees. 
% \methodname{} fits each tree using the complete training data set.
% This is common practice for Extra Trees where the diversity of trees arises from the randomness in the split \citep{geurts2006extremely}.

\begin{remark}
    Many non-private random forest algorithms train each tree on a subsample of the data using bootstrap sampling or sampling without replacement \citep{pedregosa2011scikit, athey2019generalized}, which improves diversity and reduces overfitting. In contrast, \methodname{} uses the full dataset for each tree, with randomness naturally coming from the splitting procedure, as in non-private Extra Trees \citep{geurts2006extremely}.
\end{remark}

\textbf{Leaf predictions.} The last step is to privatize predictions in leaf nodes. %\update{
For simplicity, we focus on classification in the main text and defer the extension to regression, along with some illustrative experiments, to \Cref{app:regression}.%}
% \footnote{In this case, privatizing leaf predictions amounts to a set of private averaging queries.} %but the extension to regression is straightforward.
We consider two mechanisms for privatizing leaf predictions.
First, we support the exponential mechanism to select a single class label for each leaf. The prediction for the entire forest then consists of a majority vote across trees.
Second, we support adding noise to class counts to obtain a private approximation of the class proportions, and average these proportions across trees.
We did not observe significant performance differences between the two approaches in initial exploratory experiments; majority voting is used in our experiments. We perform privacy accounting in zCDP.
Specifically, we compute the largest $\rho$ such that any $\rho$-zCDP algorithm satisfies $(\varepsilon_2, \delta_2)$-DP (see \Cref{lem:zCDP-conversion}).
We then apply a $(\rho/k)$-zCDP mechanism independently to each leaf, which yields an overall $\rho$-zCDP guarantee for the forest because each data point influences only a single leaf per tree. Additional implementation details are provided in \Cref{app:implementation}.

% For the leaf prediction we perform privacy accounting in zCDP. 
% We compute the largest $\rho$ such that any $\rho$-zCDP algorithm satisfies $(\varepsilon_2, \delta_2)$-DP (see \Cref{lem:zCDP-conversion}).
% We then apply a $\rho/k$-zCDP algorithm to each leaf, which implies $\rho$-zCDP for the forest since each data point affects only a single leaf in each tree. 
% We considered two mechanisms for privatizing leaves. 
% We support the exponential mechanism for selecting a single label for each leaf, and we support adding noise to class counts to approximate a probability distribution over class labels.
% We discuss technical details further in \Cref{app:implementation}.
% We did not observe significant performance differences between the two approaches in initial exploratory experiments. 
% We use the majority vote implementation for the experiments in the next section.

\textbf{Choosing hyperparameters.}
% \methodname{} has four hyperparameters: $\tau$, $h$, $k$, and the share of the privacy budget spent on tree construction.
%the threshold $\tau$, the maximum depth $h$, the number of trees $k$, and the share of the privacy budget spent on tree construction.
We recommend a simple, principled rule for setting $\tau$, $\Delta$, and $\sigma$ that requires no hyperparameter tuning and consistently delivers strong performance across our experiments. Specifically, given a privacy budget $(\varepsilon, \delta)$, a number of trees $k$, and a maximum height $h$, we choose the valid pair $(\Delta, \sigma)$ satisfying the inequality in \Cref{thm:build-forest} that minimizes $\Delta$, and set $\tau = \Delta + 1$.
% We recommend $\tau = 1 + \Delta$, the smallest value allowed for large trees by \Cref{thm:build-forest}. 
This threshold is typically higher than common used values in the non-private setting. Increasing $\tau$ further may be beneficial for very large datasets or when using weaker privacy guarantees.
%Further increasing $\tau$ may slightly increase the efficiency of the algorithm as outlier points are pruned more quickly while coarsening the tree. With this rule, $\tau$ is not tuned but is determined by $\varepsilon_1$, $\delta_1$, $k$, and $h$.
%The maximum depth is cheap in terms of privacy because the noise grows only with $\sqrt{1 + \floor{\log_2 h}}$. %: increasing it from $100$ to $1000$ increases $\sigma$ by only about $20\%$, compared to a factor $\sqrt{10}$ for the top-down baseline. 
Since pruning adaptively determines the effective depth of each tree, the maximum depth $h$ should be set sufficiently large so that it rarely becomes a limiting factor. In our experiments, setting $h = 100$ works well across all datasets.
%The number of estimators has implications on the other parameters: $\Delta$, and thus $\tau$, grows roughly as $\sqrt{k}$ (\Cref{lem:theorem-simpler}), and the leaves of each tree receive only a $1/k$ share of the leaf budget. 
Unlike the non-private setting, the number of trees $k$ directly affects the quality of individual trees, since the privacy budget must be divided among all trees.
Larger forests therefore consist of coarser and noisier trees.
%while larger datasets can afford more trees. 
Nevertheless, we empirically find that accuracy remains stable across a broad range of forest sizes. %(TODO: refer to results from the rubuttal?).
%\methodname{} is more sensitive to the privacy budget split. However, the choice of the split fraction does not strongly depend on the choice of dataset. 
%We observed stable performance for allocations between 70\% and 85\% for tree construction.
Finally, one must choose how to split the privacy budget between tree construction and leaf queries.
We observe stable performance across a broad range of allocations. 
For classification, we recommend allocating most of the privacy budget to tree construction. For regression, leaf queries are arguably more challenging to estimate, so we allocate a larger share of the budget to them~(see \Cref{app:regression}).
%Since hyperparameter tuning on private data incurs additional privacy cost~\citep{Papernot022}, 
%We stress that all experiments use the same depth, budget split, and threshold without per-dataset tuning, and we recommend these as defaults.

\section{Experiments}
\label{sec:experiments}

\textbf{Baselines.} We compare \methodname{} against several differentially private random forest methods, including three randomized tree-based approaches (IBM's Differential Privacy Library \texttt{DiffPrivLib}~\citep{holohan2019diffprivlibibmdifferentialprivacy}, \texttt{SNR}~\citep{FletcherI15_random} and \texttt{Smooth Sensitivity}~\citep{fletcher2017differentially}), and one greedy tree method (\texttt{Suihkonen}~\citep{Suihkonen2023DP_RF}). 
As non-private baselines, we use \texttt{sklearn}'s Extra Trees and Decision Tree classifiers \citep{pedregosa2011scikit}. 
%For a non-tree based baseline, we implement a Logistic Regression that satisfies $(\varepsilon, \delta)$-DP, trained via DP-SGD \cite{abadi_deep_2016}. For accounting, we use the numerical accountant of \citet{doroshenko_connect_2022}, implemented in the \texttt{dp-accounting} library \cite{noauthor_differential-privacypythondp_accounting_nodate}. 
We evaluate all methods across privacy budgets $\varepsilon \in \{0.5, 1, 2, 4, 8\}$ with $\delta = 10^{-6}$ for \methodname{}.
We note that some baselines have privacy issues, leading to weaker-than-claimed guarantees; this appears to be a recurring issue in prior work on DP random forests~(see \Cref{app:privacy-concerns}).\looseness=-1

\textbf{Datasets.} We evaluate performance on the \texttt{Adult} dataset~\citep{adult} and four binary classification tasks from \texttt{Folktables} (California 2018)~\citep{ding2021retiring}.
% While \texttt{Adult} is widely used, it has been criticized as overly standard~\cite{ding2021retiring}, motivating our additional benchmarks.
We exclude ACSMobility due to uniformly poor performance across all methods, including non-private baselines. Dataset statistics are summarized in \Cref{tab:dataset-stats}. Each experiment is repeated 5 times and results are averaged.

\begin{table}[h]
\centering
\caption{Summary statistics for the Adult and Folktables (CA) datasets.}
\label{tab:dataset-stats}
\begin{tabular}{lccccc}
\hline
\textbf{Dataset} & \textbf{Features} & \textbf{Categorical} & \textbf{Classes} & \textbf{Majority} & \textbf{Datapoints} \\ \hline
\texttt{Adult} & 14 & 8 & 2 & 76.1 \% & 48,842 \\
\texttt{ACSIncome (CA)} & 10 & 4 & 2 & 58.9 \% & 195,665 \\
\texttt{ACSEmployment (CA)} & 16 & 10 & 2 & 54.4 \% & 378,817 \\
\texttt{ACSPublicCoverage (CA)} & 19 & 10 & 2 & 63.1 \% & 138,554 \\
\texttt{ACSTravelTime (CA)} & 16 & 7 & 2 & 51.5 \% & 172,508 \\ \hline
\end{tabular}
\end{table}

\textbf{Hyperparameters.}  We use an 80-10-10 train/validation/test split. Hyperparameters for \texttt{DiffPrivLib} and non-private baselines are tuned on the validation set, without accounting for additional privacy cost. The performance of \texttt{DiffPrivLib} is sensitive to these choices. For \texttt{SNR}, \texttt{Smooth Sensitivity}, and \texttt{Suihkonen}, we use the settings recommended in the original papers.
For \methodname{}, %\update{\sout{to demonstrate its robustness,}} 
we fix a maximum tree depth of 100 across all experiments. 
We train 30 trees on \texttt{Adult} and 50 trees on all \texttt{Folktables} tasks. We allocate 75\% of the privacy budget to tree construction $(\varepsilon_1 = 0.75\varepsilon)$.
We set $\tau = 1 + \Delta$ in all experiments.

%To showcase that our algorithm is not sensitive to its hyperparameters, we use 75\% of the privacy budget for splitting and a maximum depth of 100 across all benchmarks. We trained 30 trees for \texttt{adult} and 50 trees for all the \texttt{folktables} prediction tasks. We tuned the hyperparameters for DiffPrivLib.% and for the other baselines, we used the configurations given in their papers, because tuning was infeasible due to their runtime. 
%For \texttt{SNR}, \texttt{Smooth Sensitivity}, and \texttt{Suihkonen} we use the hyperparameters suggested in the papers. 
%We refer to \Cref{app:additional-experiments} for more details. % \cl{TODO: Add more details in appendix}
%\de{please double check if this is accurate, TODO: put the hyperparameter configs of the baselines in the appendix (?)}

\textbf{Results.} As shown in \Cref{fig:experiment-results} (and \Cref{fig:appendix-experiments} in the appendix), \methodname{} consistently and significantly outperforms prior work across all evaluated privacy budgets. 
\texttt{DiffPrivLib} and \texttt{Smooth Sensitivity} achieve reasonable performance at small $\varepsilon$ but do not improve as privacy relaxes.
This is partly due to their reliance on data partitioning for leaf predictions, which allows them to perform accurate leaf predictions at small privacy budgets but shifts the burden to random tree construction.
We also note that \texttt{Smooth Sensitivity} handles categorical features differently than \texttt{DiffPrivLib}: it splits on \emph{all} values of a categorical feature instead of $1$, which provides an advantage on \texttt{Adult} but does not generalize well to the other datasets.
%\texttt{Smooth sensitivity} achieves decent performance on \texttt{Adult} already with low privacy budget, but does not improve for higher $\varepsilon$. 
%This can be explained by their handling of categorical variables where they split on every possible value at once and sacrificing diversity across different trees \cite{fletcher2017differentially}.  %\de{it seems like we have the same citation twice in the .bib file for this}. \cl{Fixed}
% The accuracy of \texttt{SNR} increase for higher $\varepsilon$ but remains much lower than \methodname{}.
%The accuracy of \texttt{SNR} tends to increase with higher $\varepsilon$, while \texttt{Suihkonen} and \texttt{Diffprivlib} both neither achieve good performance nor improve with less privacy. 
%\texttt{SNR} improves with increasing $\varepsilon$ but remains substantially below \methodname{}. \cl{TODO: Update plots and SNR conclusion here}
\texttt{SNR} outperforms the other baselines for some datasets but remains substantially below \methodname{}.
\texttt{Suihkonen} matches previously reported results on \texttt{Adult}~\citep{Suihkonen2023DP_RF} but performs poorly on the \texttt{Folktables} tasks.

While a gap remains to non-private \texttt{Extra Trees}, \methodname{} often outperforms a non-private greedy decision tree even at moderate privacy levels. Overall, \methodname{} establishes a new state-of-the-art among DP random forest methods. Additional results are provided in \Cref{app:additional-experiments}.
Source code is available at \url{https://github.com/daviderb/Lumberjack}.

% The performance of \texttt{Suihkonen} on \texttt{Adult} matches the plot from \cite{Suihkonen2023DP_RF}. However, we found the utility extremely low in our experiments on the \texttt{Folktable} benchmarks.
% We note that while there is still a gap to the non-private \texttt{Sklearn Extra Trees}, \methodname{} outperforms a single scikit-learn greedy decision tree for reasonable $\varepsilon$ on several datasets.
%We note that while there is still a gap to the non private \texttt{Sklearn Extra Trees} for reasonable $\varepsilon$, \methodname{} approaches its non-private counterpart for $\varepsilon \to \infty$ and even outperforms a single scikit-learn greedy decision tree. 
% \methodname{} therefore constitutes a new state-of-the-art in this line of work. 
%To the best of our knowledge, it is the first random forest with practical utility, partly due to privacy violations in prior work as described in \Cref{app:privacy-concerns}. 
% The remaining experiments are deferred to \Cref{app:additional-experiments}.

\begin{figure}
    \centering
    \includegraphics[width=1\linewidth]{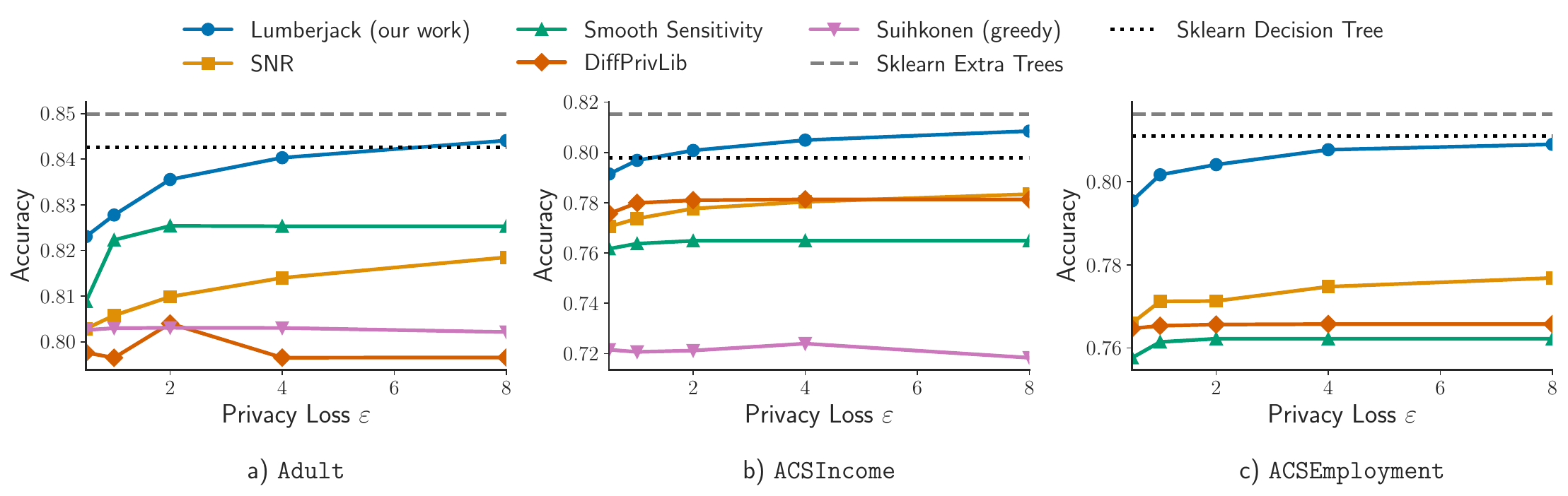}
    \caption{Test accuracy for the Adult and Folktables ACSIncome and ACSEmployment datasets. We exclude the Suihkonen baseline from ACSEmployment because its runtime was too high.} % with $\delta=10^{-6}$.}
    \label{fig:experiment-results}
\end{figure}

%\section{Conclusion, Limitations, and Open Problems}
\section{Conclusion}
\label{sec:conclusion}

Despite the popularity of random forests in non-private machine learning, their differentially private counterparts remain relatively underdeveloped and, in many cases, problematic in both correctness and utility. Our work takes a step toward addressing these issues.
% We have highlighted serious issues with prior work on differentially private random forests both in terms of correctness~(\Cref{app:privacy-concerns}) and utility.
We introduce \methodname{}, a new DP random forest method based on a combination of randomized tree construction and data-dependent pruning %structured splitting 
based on a novel heavy hitters algorithm which may be of independent interest. Empirically, \methodname{} sets a new state-of-the-art in DP random forest methods.

%While our experiments show that \methodname{} significantly outperforms prior work, we note that our focus has been to design an effective DP variant of Extra Trees \citep{geurts2006extremely}. 
%It is likely that techniques for greedy tree building algorithms similar to \cite{breiman2001random} can also be improved. % \cl{I removed this. We needed to discuss limitations, which is why is added it. But I realized that we already discuss a limitation for the lower bound}
% We believe that DP random forests is an underexplored area for privacy-preserving machine learning, and 
% we hope that our work will inspire further research in this direction.

We believe that designing high-utility differentially private greedy random forests remains an open and promising direction. Achieving this goal will likely require more careful integration of advanced tools from the DP literature, such as refined composition techniques, tighter privacy accounting, and problem-specific mechanisms that better exploit the structure of greedy splitting procedures.

Overall, we hope our work helps clarify the design space and current status of DP random forests, and encourages further research on methods that more closely match the performance of their non-private counterparts while maintaining rigorous privacy guarantees.

%\placeholder{We improve one kind of DP forest. We (hopefully) show that DP random forests can be more useful than shown in prior work. Clear future directions such as better splits, but we don't need to be too specific here. Open problem: stronger lower bound for heavy hitters in trees - can we get matching upper and lower bounds?}

\section*{Acknowledgements}
The work of David Erb was carried out while visiting the PreMeDICaL Inria team. 
The work of Christian Janos Lebeda, Tudor Cebere, and Aurélien Bellet is supported by grant ANR-20-CE23-0015 (Project PRIDE) and the ANR 22-PECY-0002 IPOP (Interdisciplinary Project on Privacy) project of the Cybersecurity PEPR. The work of Tudor Cebere is also supported by a Google PhD Fellowship in Privacy, Safety, and Security. The work of Christian Janos Lebeda is also supported by the Carlsberg Foundation, grant CF25-1559. This work was performed using HPC resources from GENCI–IDRIS (Grant
2023-AD011014018R2).

\bibliographystyle{plainnat}
\bibliography{bibliography} 

%\newpage
\appendix

\section{Additional Preliminaries}
\label{app:additional-preliminaries}

Here we present standard technical results from differential privacy.

Any $\rho$-zCDP guarantee can be converted to $(\varepsilon, \delta)$-DP guarantees as follows.

\begin{lemma}[{\citep{BunSteinke16}}]
\label{lem:zCDP-conversion}
    Let $\mathcal{M}: \mathcal{X}^{*} \rightarrow \mathcal{Z}$ be a randomized mechanism that satisfies $\rho$-zCDP. Then, $\mathcal{M}$ satisfies $(\varepsilon, \delta)$-DP. for all $\varepsilon \ge \rho$ and
    \begin{align*}
        \delta = \frac{2e^{-(\varepsilon - \rho)^2 / 4\rho}}{1 + \frac{\varepsilon - \rho}{2\rho} + \sqrt{(1 + \frac{\varepsilon - \rho}{2\rho})^2 + \frac{4}{\pi \rho}}} \,.
    \end{align*}
\end{lemma}

The consecutive execution of differentially private mechanisms yields another differentially private mechanism. We use the following composition theorems.

\begin{lemma}[{\citep{DworkR14}} Basic Composition]
\label{lem:basic-composition}
    Let $\mathcal{M}_1: \mathcal{X}^{*} \rightarrow \mathcal{Z}_1, \mathcal{M}_2: \mathcal{X}^{*} \times \mathcal{Z}_1 \rightarrow~\mathcal{Z}_2$ be randomized mechanisms that satisfy $(\varepsilon_1, \delta_1)$-DP and $(\varepsilon_2, \delta_2)$-DP respectively, then the composed mechanism $\mathcal{M}(D) = \mathcal{M}_2(D, \mathcal{M}_1(D))$ satisfies $(\varepsilon_1 + \varepsilon_2, \delta_1 + \delta_2)$-DP.
\end{lemma}

\begin{lemma}[{\citep{BunSteinke16}} zCDP Composition]
\label{lem:zCDP-composition}
    Let $\mathcal{M}_1: \mathcal{X}^{*} \rightarrow \mathcal{Z}_1, \mathcal{M}_2: \mathcal{X}^{*} \times~\mathcal{Z}_1 \rightarrow~\mathcal{Z}_2$ be randomized mechanisms that satisfy $\rho_1$-zCDP and $\rho_2$-zCDP respectively, then the composed mechanism $\mathcal{M}(D) = \mathcal{M}_1(D, \mathcal{M}_2(D))$ satisfies $\rho_1 + \rho_2$-zCDP.
\end{lemma}

A key property of differential privacy is that the usage of the output of a private mechanism in any computation cannot make the output less private.

\begin{definition}[{\citep{DworkR14}} Post-Processing] A randomized mechanism $\mathcal{M}: \mathcal{X}^{*} \rightarrow \mathcal{Z}$ with privacy guarantee $g$ is immune to post-processing if for any randomized mapping $f: \mathcal{Z} \rightarrow \mathcal{W}$, $f \circ \mathcal{M}: \mathcal{X}^* \rightarrow \mathcal{W}$ satisfies $g$. 
\end{definition}

\begin{lemma}[{\citep{DworkR14, BunSteinke16}}] Every mechanism that satisfies a $(\varepsilon,\delta)$-DP or $\rho$-zCDP guarantee is immune to post-processing.
\end{lemma}

The Gaussian mechanism as presented below is one of the most used DP techniques.

\begin{lemma}[\citep{BunSteinke16,Balle18}, The Gaussian mechanism]
    \label{lem:gaussian-mech}
    Let $q \colon \mathcal{X}^* \rightarrow \mathbb{R}^d$ be a set of queries with $\ell_2$ sensitivity $\Delta q \coloneq \max_{D \sim D'} \|q(D) - q(D')\|_2$. 
    Then the mechanism that outputs $\tilde q(D) = q(D) + Z$ where $Z \sim \mathcal{N}\left(0, (\Delta q)^2 \sigma^2 I_d\right)$ satisfies $(2/\sigma^2)$-zCDP and $(\varepsilon,\delta)$-DP for 
    \[
    \delta \geq \Phi\left( \tfrac{1}{2\sigma} - \varepsilon\sigma \right) - e^\varepsilon \Phi\left( - \tfrac{1}{2\sigma} - \varepsilon\sigma \right) \,.
    \]
\end{lemma}

\section{More Details of Related Work for Differentially Private Heavy Hitters}
\label{app:related-heavy}

Here we present an expanded discussion of the related work on differentially private heavy hitters. % \Cref{sec:related-heavy-hitters}.

A key technical component of our work is a novel algorithm for the heavy hitters problem for hierarchical data.
For this problem, we are given a tree $\T$ with height $h$. Each data point increase the count of a single leaf by 1, and each internal node stores the sum of counts for all child nodes.\footnote{We focus on binary trees in the presentation but our algorithm works with any branching factor. More generally, our heavy hitters algorithm is applicable to any setting where adding a data point only changes the value of nodes that are all part of the same path from root to leaf and the value of a parent node is never smaller than any child value.}
We want to privately detect all nodes with count above some threshold $\tau$. Notice that the root of $\T$ is always heavy and the set of heavy hitters form a tree with the same root (except for the special case where no nodes are heavy). We exploit this structure in our algorithm.

This problem has connections both to differentially private quantile estimation and sparse histograms. 
If we restrict our view to a single path from the root to the leaf, we want to find the first node with count below the threshold. 
This is equivalent to a quantile query (also known as threshold queries).
In the (discrete) quantile problem, we have $n$ data points $x_i \in \{1,\dots,h\}$.
The goal is to find the largest $q$ such that 
$\vert\{i \in [n]: x_i \leq q\}\vert \leq \tau$ for some threshold $\tau$.
Standard approaches include private selection using the exponential mechanism which has rank error $O(\log h/\varepsilon)$ (or $O(\log h /\sqrt{\rho})$ under $\rho$-zCDP)~\citep{kaplan2022differentially} and noisy binary search which has rank error $O(\sqrt{\log h \log\log h /\rho})$ for $\rho$-zCDP~\citep{HuangLY21}. Other techniques achieve better dependence on $h$~(e.g. \cite{BunNSV15,Cohen0NSS23}), but have larger constants and are only preferred for very large $h$. %\looseness=-1

If we instead restrict out view to a single layer of the tree, the problem is equivalent to detecting heavy hitters in a histogram. 
For deeper layers in the tree, this histogram must be sparse, since the number of nodes exceeds the number of data points. 
In the sparse histogram problem, we have $n$ data points $x_i \in [d]$ for some $d \gg n$. 
We want to privately estimate $\vert \{ i \in [n] : x_i = j\} \vert$ for all $j \in [d]$.
Variants of this problem have been thoroughly studied in the differential privacy literature~\citep{Korolova09,CormodePST12,Gotz2012,BalcerV19,WilsonZLDSG20,ALP22,WilkinsKZK24,LebedaR25,KerschbaumLeeWu25}. 
The optimal error for this problem is $O(\min(\log(d), \log(1/\delta))/\varepsilon)$. We rely on a lower bound from this literature to derive a lower bound for our setting. 
As such, the heavy hitters problem we consider generalizes both quantile queries and sparse histograms. In fact, we leverage and adapt techniques from both problems in our algorithm.

A related problem is the task of releasing counts for hierarchical data. 
The straightforward approach samples independent Gaussian noise scaled by $\sqrt{h}$ to each count (here $h$ is the maximum height of a tree and thus the $\ell_2$ sensitivity is $\sqrt{h}$). 
Several papers \citep{honaker2015efficient, optimizingAttribution2023, ghazi2026denoising} leverage redundancy in the tree structure to achieve better utility by post-processing the noisy queries.
An alternative approach to exploiting redundancy in queries is to directly add correlated noise using the factorization mechanism framework~\citep{LiMHMR15,EdmondsNU20}.
Assuming that all nodes on the same layer in $\T$ have the same branching factor, the tree structure can be encoded as $h$ marginal queries. Marginal queries have received significant study in differential privacy and recent results provide optimal factorizations \citep{XiaoHZK23, LNT25, he2026accurate}.
Both the postprocessing techniques and the factorization framework lead to useful constant factor improvements, but the noise still scales with $\sqrt{h}$.
Since there are $2^h-1$ noise samples, one for each node, the expected largest noise sample has magnitude $O_{\varepsilon,\delta}(\sigma\sqrt{\log(2^h)}) = O_{\varepsilon,\delta}(h)$. 
In many use cases, $h$ is relatively small; for instance, the data used for the US 2020 Decennial Census~\citep{censusTopDown, he2026accurate} has 6 levels (Country, State, County, Tract, Block Group, and Block), so multiplying noise by a factor $\sqrt{h}=\sqrt{6}\approx2.45$ is acceptable for these problems.
We aim to support significantly deeper trees, and thus require a much lower dependency on the tree height.\looseness=-1

Despite a lot of prior work in differential privacy both on heavy hitters estimation and release of hierarchical data, there is limited work on the intersection. 
\cite{GhaziK0M023} consider the same heavy hitters problem as us as part of their work (see Problems 4.2 and 4.4 therein). They reduce the dependency on the height to $O(\sqrt{\log h})$. However, they introduce a $\sqrt{n}$ term, which is typically much higher than $\sqrt{h}$.
The most relevant related work is by \cite{BiswasCKSZ24} who consider the task of privately releasing Hierarchical Heavy Hitters (HHH), a generalization of the heavy hitters problem. 
In the HHH problem, the count of a hierarchical heavy hitter is not included in the sum for the parent node. 
Note that this is a strictly harder problem, since we can recover the tree of heavy hitters by post-processing the HHH set.
\cite{BiswasCKSZ24} study the problem in both the static and the streaming settings. 
In the static setting they propose an algorithm inspired by the Sparse Vector Technique (SVT)~\citep{DworkR14,lyu2016understanding} with noise scale independent of the height of the tree.
The error scales only with $O(\max(\log(nh),\log(h/\delta))/\varepsilon)$. 
Since their setting is strictly harder than our heavy hitters problem, we would ideally use their algorithm for detecting heavy hitter.
Unfortunately, we have discovered a subtle mistake in their privacy proof, as we discuss in \Cref{app:HHH-bug}. 
%\update{
The authors recently uploaded a revised version of their paper to arXiv, in which they correct the mistake~\citep{biswas2026}. The error of their corrected mechanism is $\tilde{O}(\log(1/\delta)\log(h/\delta)/\varepsilon)$, which incurs an additional $O(1/\delta)$ factor compared to their initial result. Consequently, their bound is worse than ours in terms of both $h$ and $\delta$. Moreover, their technique has larger constants than ours, as shown in \Cref{tab:comparison}. Nevertheless, the problem they consider is strictly more general, and determining whether their upper bound can be improved remains an interesting open question.%}

Two recent papers \citep{BernardiniBGS25, GuoHollandWu26} study the problem of privately identifying frequent substrings. 
Their techniques relies on a tree-based data structure. 
Although the setting is not directly comparable to the problem we consider in this paper, we believe (adaptations of) our algorithm may be relevant for identifying frequent substrings under approximate DP.

\section{Additional Experimental Results and Details}
\label{app:additional-experiments}

\paragraph{Results on additional datasets.}
\Cref{fig:appendix-experiments} shows results for datasets that were left out of the main text. 
Similar to \Cref{fig:experiment-results}, we outperform all the private competitors.

\begin{figure}[h]
    \centering
    \includegraphics[width=0.95\linewidth]{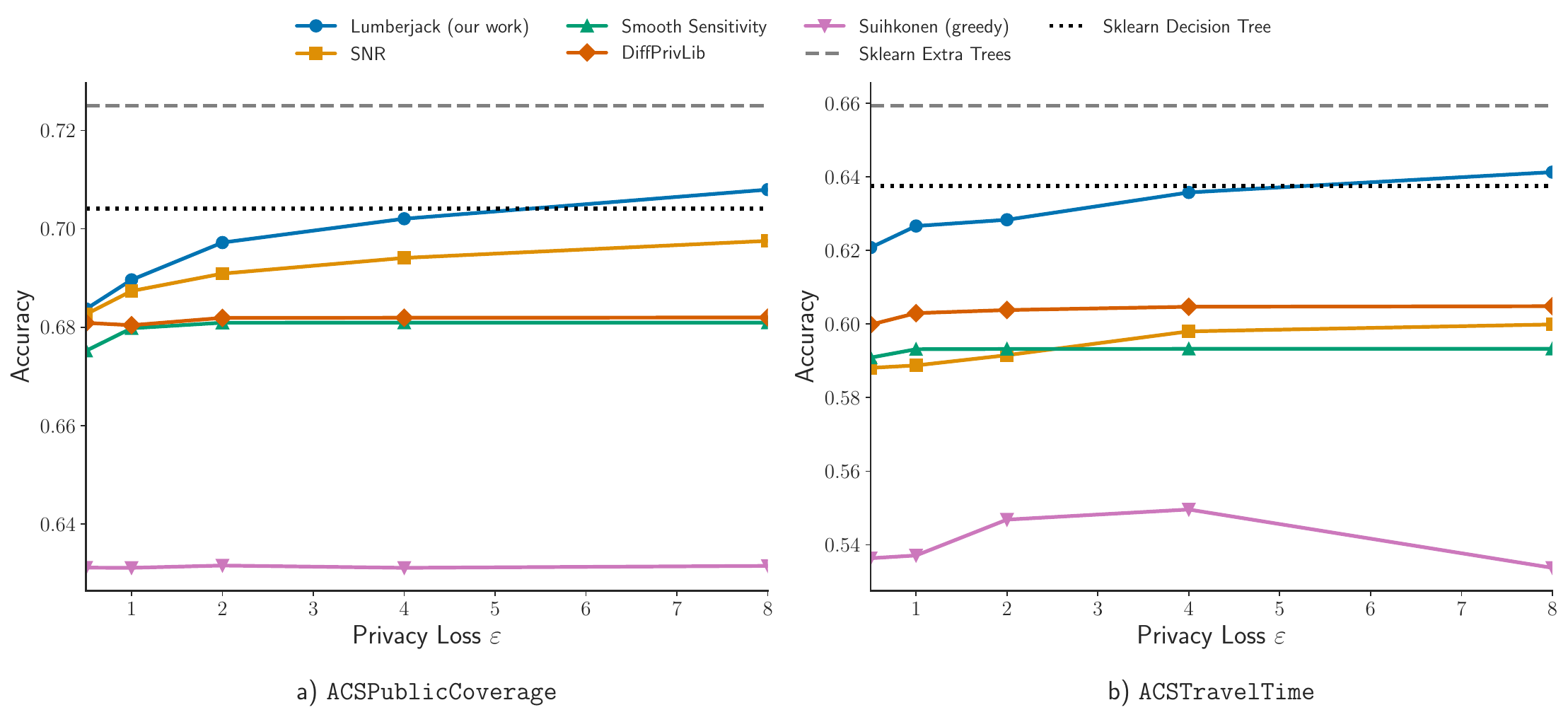}
    \caption{Experimental results for two additional Folktables classification tasks. %We exclude the Suihkonen baseline from ACSEmployment due to technical issues.
    }
    \label{fig:appendix-experiments}
\end{figure}

\paragraph{Hyperparameters for baselines.} The \texttt{SNR} implementation computes most hyperparameters based on the input. 
They do not support continuous features directly, so the domain has to be discretized. We set the number of bins to 5 reflecting parameters used in the paper~\citep{FletcherI15_random}.

For \texttt{Smooth Sensitivity}, we only set the number of trees since the depth is computed using the number of features. We train 30 trees for all benchmarks because they used 30 trees for their experiments with \texttt{Adult}~\citep{fletcher2017differentially}.
Note that we used the original implementation which contains a known privacy bug. The bug is fixed in the \texttt{DiffPrivLib} variant of the algorithm, but we include the original source code as a baseline because the implementations handle categorical features in 2 different ways.

For \texttt{DiffPrivLib}, we tuned hyperparameters using the validation set. 
Strictly speaking, one should spend part of the privacy budget for this tuning, since the choice of hyperparameters can reveal information about the dataset~(See \cite{Papernot022}).
However, we do not include this in the privacy cost to mimic the best case scenario where an analyst picked good hyperparameters.
This implementation often predicts the majority class for all data points if we pick bad parameters. 
We consider forest sizes in increments of 5 trees and explore tree depth up to 19. The typical selected parameters were 10-20 trees and depth 15-17. 
Note that increasing the search to deeper trees quickly becomes infeasible due to memory usage since \texttt{DiffPrivLib}, unlike \methodname{}, stores all empty nodes in the tree. For example, at depth 20, there are more than a million nodes in a tree so most of the memory usage is dedicated to nodes that contain no data points.
For \texttt{DiffPrivLib}, we did not use one-hot encoding for categorical features. The implementation does not support weighted feature sampling, and the classifier performs much worse with one-hot encoding. Instead, we set the value to the index in the order of the original data.

For \texttt{Suihkonen}, we used the same hyperparameters that the original work used for \texttt{Adult}, which is a forest of 15 trees with depth 7. 
When splitting on categorical featuresn they split on a set rather than a single category. 
They use a parameter to determine the size of subsets, which we set to 8.
%They use a parameter to determine subsets of categories to explore when splitting categorical attributes.
%We set this parameter to 8.
The accuracy on \texttt{Adult} roughly matches the plot from~\citep{Suihkonen2023DP_RF}. 
We did not change the hyperparameters for the \texttt{Folktables} tasks except for only considering "one vs. all" splits on categorical attributes. 
This was required for runtime purposes. 
The choice of this parameter affects the performance on these datasets, but extensive hyperparameter tuning for this algorithm was impractical since certain choices are computationally expensive.
Note that we did not include the implementation of \cite{Suihkonen2023DP_RF} for ACSEmployment because the runtime was too high for this dataset.\looseness=-1
% Nevertheless, based on the other experiments we are confident that we would still outperform the approach if those issues were fixed.

We found privacy concerns in all papers from the greedy decision tree paradigm (see Appendix~\ref{app:privacy-concerns}), making it unclear which technique should be regarded as the prior state of the art.
We use \texttt{Suihkonen} as a representative technique from this paradigm. As it performs very poorly on folktables experiments, we considered including an additional baseline based on greedy decision trees. 
However, we did not find any suitable candidate among existing literature. 
Early works such as \cite{PatilS14_random_forest,FletcherI15_greedy} are known to perform poorly in general, and they do not directly support continuous features. 
The runtime of \cite{consul2021differentially} is exponential in the size of categorical features, which is infeasible for several of the datasets we consider. 
Many of the other papers do not provide any source code.
% , and anyway contain clear privacy issues (see Appendix~\ref{app:privacy-concerns}).
% We decided not to implement those algorithms from scratch, since they contain clear privacy issues.
% We included \texttt{Smooth Sensitivity} and \texttt{Suihkonen} despite privacy concerns to increase the number of baselines.
% We identified privacy concerns in all papers based on the greedy decision tree paradigm, making it unclear which technique should be regarded as the prior state of the art (see Appendix~\ref{app:privacy-concerns}).
% Moreover, many of these works do not provide source code.
% We decided not to implement those algorithms from scratch, since they contain clear privacy issues.
% Nevertheless, we included \texttt{Smooth Sensitivity} and \texttt{Suihkonen} to increase the number of baselines despite their privacy concerns.

\paragraph{Toy Example in \Cref{fig:moons-toy-example}.} The dataset is constructed by combining the \texttt{make\_moons} and \texttt{make\_blobs} dataset generators from scikit-learn. We create 10000 data points in total, 5500 red points, 4000 blue points, and 500 green points. 
We add an offset to the moons such that the feature bounds are non-negative.
The center of the green outlier blob is distance 1 from the upper bound of both features.
We set the bounds to $[0,5]$, $[0,50]$, and $[0,5000]$ for Feature 1 and $[0,4]$, $[0,40]$, and $[0,4000]$ for Feature 2.
We train 25 trees and use majority voting for classification. 
We tested various depths for the fully random trees for the $[0,5] \times [0,4]$ setting and found that depth 10 splits the classes fairly well. When the bounds increase, most of the feature space is empty. As a result, many of the random splits create empty nodes, and the fully random trees completely fail to separate the moons.
For \methodname{}, we set $\varepsilon_1=0.75\varepsilon$ and use a max depth of 100. 
Although many of the early splits in each trees create empty leaves for the extreme $[0,5000] \times [0,4000]$ setting, our algorithm eventually locates the relevant part of the feature space because it adapts to the number of data points.

\section{Technical Issue with SVT-based Argument by \cite{BiswasCKSZ24}}
\label{app:HHH-bug}

% \fbox{%
% \parbox{.98\linewidth}{%
% \color{red}
% \textbf{Disclaimer to NeurIPS reviewers:} We have notified the authors of \cite{BiswasCKSZ24} about the error. They acknowledged the problem and are exploring possible variations to the algorithm to resolve the issue. As of the NeurIPS submission deadline, no solution is known.
% }%
% }

%\placeholderdavid{TLDR: SVT used the fact that queries are correlated, for the unrelated nodes they use what is essentially parallel composition. The arguments are correct in isolation but incompabitel }

In this section, we describe a subtle mistake that we discovered in the privacy analysis of an algorithm designed by \cite{BiswasCKSZ24} to find Hierarchical Heavy Hitters.
At a high level, their algorithm traverses the tree in bottom-up order and privately compares the count of each node against a noisy threshold. 
Crucially, the noise for this comparison does not scale with $h$ as the algorithm use a technique inspired by the popular Sparse-Vector-Technique (SVT) \citep{DworkR14, lyu2016understanding}.
\cite{BiswasCKSZ24} consider a more general version of the problem that we address in \Cref{sec:pruning-algorithms}, their technique could in principle be used to implement a variant of \methodname{}. However, as we show below, the strong privacy properties of SVT do not appear to transfer directly to trees. 
Despite our efforts, we have not been able to resolve the issue with their algorithm (apart from the trivial fix of increasing the  noise scale by a factor $O(\sqrt{h})$).

%In this section, we describe a subtle mistake in the privacy analysis of an algorithm designed by \cite{BiswasCKSZ24} to find Hierarchical Heavy Hitters. Without this error, we would have been able to improve our tree construction by using their algorithm that adds noise which scales independently of the height of the tree.

%\cl{David: A general writing tip - you need to introduce a section. You dive straight into the technical details here. The first line references Algorithm 1 and Theorem 3.1. But it's not clear that this is an algorithm and theorem from their paper. Remember, we cannot assume that the reader is familiar with \cite{BiswasCKSZ24}. Clear writing is difficult. You have to include all details but remain concise.}

At a high level, the proof of the $(\varepsilon, \delta)$-DP privacy guarantee for Algorithm 1 of the paper \citep[Theorem~3.1]{BiswasCKSZ24} proceeds as follows for any two neighboring datasets $X$ and $X' = X \cup x'$: 
%For any fixed output, the probability of the output occurring under any two neighboring datasets $X$ and $X' = X \cup x'$, respectively.
%Using the structure of the problem, the proof separates the tree into 
Separate the nodes of the tree into 
(i) the path that is affected by adding or removing the element $x'$ and 
(ii) the remaining part of the tree. 
For the former, they apply an argument similar to the proof of the Sparse-Vector-Technique (SVT) \citep{DworkR14, lyu2016understanding}. Since they add a noise sample to the threshold which \emph{correlates} all these queries, they can output any number of "negative" comparisons ($\bot$) on the path at no additional privacy cost. They ensure that each data point is only part in one "positive" query ($\top$), which ultimately achieves a noise scale that is independent of the height of the tree.
For the latter, they point out that the output distributions for these nodes are unaffected by the change in datasets.

We argue that each analysis on it's own is correct, but the two arguments are not compatible and therefore break when combined. 
On a high level the issue is that the first argument utilizes the fact that the entire output sequence of their algorithm is slightly correlated to achieve stronger privacy, while the second argument relies on parallel composition which requires separate parts of the output to be completely independent. Thus we cannot benefit from the first privacy properties of SVT while simultaneously ignoring the effect on the other part of the output. 
We point out concrete parts of the proof that breaks in more details below.

The proof in the paper divides the nodes into three groups, but, to point out the issue in the proof, it is sufficiently to only consider the partition into the affected ("Active" nodes in the paper) and the unaffected ("Unrelated" and "After" nodes in the paper) nodes. Equation 9 in the paper then writes the probability of observing any output $\vec a = (a_1,\dots,a_m)$ as
\begin{align*}
    \prod_{i=1}^m p(\beta_i, a_i) = \prod_{i \in I_{\text{affected}}} p(\beta_i, a_i) \cdot \prod_{i \in I_{\text{unaffected}}} p(\beta_i, a_i) \,,
\end{align*}
where $m$ is the number of nodes in the tree, $\beta_i$ is the random variable of the value of node $i$ in the output of the mechanism using input $X$ and $a_i$ is the observed output. $p(\beta_i, a_i)$ is shorthand notation for $\Pr\left[\beta_i = a_i \vert \vec \beta_{-i} = \vec a_{-i} \right]$ where $\vec \beta_{-i} = (\beta_1, \dots, \beta_{i-1})$ and $\vec a_{-i} = (a_1, \dots, a_{i-1})$. The nodes are indexed in a bottom-up level order, and $\beta_i$ takes on the value $\bot$ or $\top$. $\beta'_i$ will denote the corresponding random variable when the mechanism is run on $X'$. 

For the unaffected nodes, the proof in \citep{BiswasCKSZ24} states (page 20 and 21, paragraphs \emph{Unrelated Nodes} and \emph{After Nodes}) that for any $i \in I_{\text{unaffected}}$
\begin{align*}
    p(\beta_i, a_i) = p(\beta'_i, a_i).
\end{align*}
However, $\vec \beta_{-i}$ can also include $\beta_j$ with $j \in I_\text{affected}$. This means that even though $\beta_i$ and $\beta'_i$ would have the same distribution unconditionally ($\Pr[\beta_i = a_i] = \Pr[\beta'_i = a_i]$), the equality does not hold, because the conditions of the probabilities are not identical.
The noise sample added to the threshold is denoted by $\gamma$, and we have $\Pr[\beta_j = a_j \vert \gamma] \neq \Pr[\beta'_j = a_j \vert \gamma]$ for any $j \in I_{\text{affected}}$.
The value of $\gamma$ affects all queries and observing $\beta_j$ and $\beta'_j$ reveal different information about $\gamma$ which breaks the equality.

Regarding the affected nodes, the proof applies an SVT-style argument (page 22 and 23, paragraph \emph{Active Nodes}), in which an essential part is conditioning on the threshold noise $\gamma$ \citep{lyu2016understanding}. 
The proof of \cite{lyu2016understanding} leverages the fact that the noisy threshold is never revealed to avoid the cost of composition. 
Unfortunately, in this setting, conditioning on $\gamma$ in this way is not as straight-forward, since the conditions can be affected by observations $\beta_j$ with $j \in I_\text{unaffected}$. Intuitively, we learn information about $\gamma$ from all the outputs of the unaffected nodes. In the SVT argument, we do not pay the privacy cost for additional queries that return $\bot$.
However, queries for $I_\text{unaffected}$ can contain both $\bot$ and $\top$, so we cannot ignore their impact in the privacy argument.
In extreme cases for huge worst-case input we could infer the approximate value of $\gamma$ from just the unaffected queries. 
%In the SVT argument, the threshold noise follows a Laplace distribution, but after conditioning on $\beta_j$, which are dependent on $\gamma$, $\gamma \vert \beta_j = a_j$ does not follow the Laplace distribution.

It becomes clear that both directions of the privacy proof break due to the shared dependence on $\gamma$. 
It is not possible to index the nodes in a way such that this does not occur, which shows the incompatibility between the two arguments. 
% We have unfortunately not found a solution to this subtle issue. 

%It becomes clear that both directions of the proof break because conditions might include random variables belonging to the other group. It is not possible to index the random variables in a way such that this does not occur, which shows how the two argument are not compatible.

\section{Privacy Concerns in Prior Work on Differentially Private Random Forests}
\label{app:privacy-concerns}

We identified privacy concerns in many published papers on differentially private random forests. In this section, we summarize a selection of these issues. Our goal is not to provide a comprehensive review of all privacy proofs in the literature; rather, we aim to highlight recurring pitfalls. While differential privacy algorithms and implementations are known to be error-prone in general~(see, e.g., \citep{CasacubertaSVW22,CebereEDBF}), this line of work nonetheless concentrates a number of mistakes that should be readily apparent to experts in differential privacy. Some of these issues directly affect the core claims of the corresponding works, while others could likely be resolved through relatively minor modifications to the algorithms. We do not attempt to categorize the severity of the identified problems.
We also note that several papers are written in a way that makes their technical claims difficult to verify.
We hope that this discussion will help improve the rigor, reproducibility, and comparability of future work in this area.

An additional purpose of this section is to justify our decision not to include several prior works as baselines in \Cref{sec:experiments}. Due to the identified privacy concerns, it is unclear which papers should be regarded as the prior state of the art for differentially private random forests. In particular, we found at least one privacy issue in every technique based on greedy random forests. Nevertheless, our method matches or exceeds the performance claimed by these works (e.g., those using Adult typically report an accuracy of around 79-82\%).

% One of the purposes of this section is to justify our decision not to include several papers as baselines in \Cref{sec:experiments}. 
% Due to the privacy concerns, it is unclear which papers should be considered the prior state of the art for differentially private random forests. 
% In particular, we discovered at least one bug in all techniques for greedy random forest. 
% Nevertheless, we note that we beat the claimed performance. 
% Several papers only experiment with tiny datasets while those that use Adult typically report an accuracy of around 79-82\%. 

%The concerns listed below are often clear concerns that were easy to spot. This section is meant to justify our decision not to include several papers as baselines. However, we point out that we beat their claimed performance anyway so we still improve over these techniques if they addressed the privacy concerns.

\paragraph{Recurring problems.}
Several types of differential privacy violations appear repeatedly throughout this line of work. One common issue concerns the stopping criterion used during tree construction. In non-private decision trees, it is standard practice to stop splitting when all data points in a node share the same label. Under differential privacy, however, this condition cannot be checked deterministically, since doing so directly reveals information about the underlying data. Unfortunately, we found this issue to be widespread in the literature
\citep{PatilS14_random_forest, xin2019differentially,HouLMNCL19,GuanSSWD20, zhang2021random, LiuLWLLW23, vos2023differentially,Suihkonen2023DP_RF}.

Another recurring problem arises in methods that use the exponential mechanism to select split points for continuous features. The candidate split set itself must be chosen in a privacy-preserving manner. In particular, using the observed feature values directly as the only candidate split points reveals information about the exact locations of data points and therefore violates differential privacy. One possible remedy is to sample intervals between adjacent data points proportionally to their length and utility, and then select a point uniformly within the chosen interval, as proposed by \cite{kaplan2022differentially}. Nevertheless, several works instead construct a data-dependent discrete candidate set, which violates differential privacy~\citep{li2017random, GuanSSWD20}.

Finally, \cite{PatilS14_random_forest, li2017random} employ bootstrapping without properly accounting for its privacy cost, resulting in incorrect privacy guarantees.

%Some privacy violations appear multiple times in this line of work with the first one being non-private data pre-processing. Many approaches are based on the exponential mechanism selecting the best splitpoint, which requires a finite set of candidates.\cl{Not true, you can use the exponential mechanism over real values. You just loose worst-case error guarantees but it often works well in practice.} Several proposed DP random forest algorithms \citep{li2017random, GuanSSWD20, xin2019differentially, zhang2021random, LiuLWLLW23} discretize or handle the continuous attributes in a way that violates DP.\cl{Do they all claim to handle continuous values? We should be a bit more precise. If they claim to focus on discrete data and this step is only for the experiments it would be fine.} A crucial part in constructing decision trees is when to stop splitting nodes. Under DP, it is not feasible to access private data in the stopping condition. Unfortunately, we have found that this issue is wide-spread \citep{PatilS14_random_forest, HouLMNCL19, GuanSSWD20, xin2019differentially, zhang2021random, LiuLWLLW23, vos2023differentially}\cl{These all stop based on nodes with a single class? Or multiple different stopping conditions?}. \cite{li2017random, PatilS14_random_forest} use bootstrapping, without properly accounting for it, giving inaccurate privacy guarantees.

\paragraph{Specific issues in each paper.} Below we point to the location of specific issues in each paper.

\begin{enumerate}

\item \cite{PatilS14_random_forest} deterministically stop splitting when all samples in a node belong to the same class (Section IV, subsection D, point 2). This data-dependent control flow in the construction of the tree violates DP. Additionally, they do not properly account for bootstrapping which results in an inaccurate privacy guarantee (Algorithm 1, step 3). 

\item \cite{FletcherI15_greedy} calculate the local sensitivity based on noisy values instead of using data-independent global sensitivity (Algorithm 1, line 17). 

\item \cite{li2017random} use the midpoint between adjacent data points as candidates when splitting continuous features (Section III, Subsection D, Equation 5). This is heuristically better than using the data points themselves, but it still breaks DP. They also do not properly account for bootstrapping (Step 2a of the algorithm in Section III, Subsection B).

%\cite{li2017random} do not account for their non-private pre-processing (Section III, subsection D, equation 5 - they use the midpoint between features are split candidates) and also do not properly account for bootstrapping (Step 2a of the algorithm in section III).

\item \cite{HouLMNCL19} use a deterministic data-dependent stopping criterion in the splitting algorithm (Algorithm 1, Line 5). 
Quoting directly from the paper: "The condition that the decision tree stops growing
includes: 1) All samples in the node have the same classification result".
%The random forest proposed by \cite{HouLMNCL19} is not differentially private, as it uses a data-dependent deterministic stopping criterion in the splitting algorithm of decision trees (Algorithm 1, line 5 - text from paper "The condition that the decision tree stops growing
%includes:1) All samples in the node have the same classification result").

\item \cite{xin2019differentially} have data-dependent stopping conditions in their tree-building process. They do not split nodes with 10 or fewer data points (Algorithm 2, Line 1), and they deterministically stop splitting pure nodes (Algorithm 2, line 5).

\item \cite{GuanSSWD20} sort values of continuous attributes and use the average of buckets with 5 data points as candidates for the splitpoint (Algorithm 1, Line 10). The candidate set differs between neighboring datasets which breaks DP. Their implementation also deterministically does not split on pure nodes (Algorithm 1, line 4).

\item \cite{zhang2021random} deterministically stop splitting when all samples in a node belong to the same class (Algorithms 1, 2, and 3, see termination condition). % Note: Very hard to read, likely several other problems that are hard to pintpoint.

\item \cite{consul2021differentially} use data-dependent sensitivity in both the splits (Equation 6) and the leaf nodes in the case of regression trees (Section 3, paragraph \textit{Estimating leaf node parameters}).\footnote{Incidentally, we note that we could not use this algorithm as a baseline for computational reasons. They inspect all possible subsets when splitting categorical features. This is infeasible for features with many values. The native-country feature from the Adult dataset has 41 possible values so they should evaluate $2^{40}$ possible splits.}

%This algorithm is also slow for categorical features with many values. They split up the features in two subsets. They consider all subsets, so we cannot run it for e.g. adult $>2^{40}$.

\item \cite{LiuLWLLW23} deterministically stop splitting on pure nodes (Algorithm 2, Line 4).
In the second part of their algorithm they sample points with weights based on misclassification without accounting for it (Algorithm 3, Line 9).
This leads to a higher privacy loss for outlier data points.
%and reweight points based on misclassification without accounting for it (Algorithm 3, Line 9). % they also use pre-processing

\item \cite{vos2023differentially} (who focus on the setting of a single decision tree rather than building a forest) deterministically stop splitting at pure nodes (Algorithm 1, line 7). 

\item While we use \cite{Suihkonen2023DP_RF} as a baseline for greedy forests in our experiments, the code does contain some privacy violations including data-dependent stopping criteria for splitting (line 496 in the implementation), deterministically labeling pure leaf nodes (line 492 in the implementation) and sensitivity miscalibration (line 242 in the implementation).

\end{enumerate}

Finally, we note that \cite{fletcher2017differentially} originally contained a privacy error that was later corrected. The original algorithm used a variant of the exponential mechanism calibrated with \emph{smooth sensitivity}. Their random forest method was subsequently implemented in the popular open-source library DiffPrivLib~\citep{holohan2019diffprivlibibmdifferentialprivacy}.
During the development, Naoise Holohan identified a mistake in the privacy proof for the smooth sensitivity approach. As a result, the DiffPrivLib implementation was updated to use the Permute-and-Flip algorithm~\citep{McKennaS20PermuteFlip} instead, and Fletcher \& Islam later added an addendum to an updated arXiv version of the paper describing the issue.
We commend Fletcher \& Islam, and the DiffPrivLib developers, for publicly documenting and correcting the issue.

\section{Missing Proofs from Main Body}
\label{app:proofs}

Here we present proofs of technical results from the main body of the paper.

%\placeholder{TODO: Improve formatting here}

\paragraph{Proof of \Cref{lem:bin_search_sensitivity}.}

\begin{proof}
    The element $x$ is associated with a unique path $P_x$ from the root to a leaf. We want to show that at most $1 + \floor{\log_2(h)}$ nodes along $P_x$ are queried. 
    In Phase 1, the algorithm queries one node $u \in P_x$ at level $\text{mid} = \lfloor h/2 \rfloor$ (if it is not already marked).
    We show in the following that $P_x$ is queried in at most one recursive call of the algorithm, which implies the desired bound. 
    The mark of $u$ determines which recursive phase may subsequently query nodes on $P_x$:
    
    \begin{itemize}
        \item \textbf{Case 1: $u$ is marked \textit{Light}.} 
        All descendants of $u$ (the bottom half of $P_x$) are marked \textit{Light}. 
        Consequently, the recursive step containing $P_x$ is skipped in Phase 2 by the condition on Line \ref{line:condition-recursive-nonlight}. Any further queries to $P_x$ occur in the recursive call on the top part of the tree in Phase 3.
        \item \textbf{Case 2: $u$ is marked \textit{Heavy}.} 
        All ancestors of $u$ (the top half of $P_x$) are marked \textit{Heavy}.
        Because marked nodes are skipped by Line~\ref{line:condition-unmarked}, Phase 3 never queries any nodes in $P_x$. When $h > 2$, $P_x$ is part of exactly one recursive call in Phase 2 via the subtree $\T_v$ that contains $P_x$.  
    \end{itemize}  
    
    In both cases, $x$ is involved in exactly one query and at most one recursive call on a tree of height $\leq \lfloor h/2 \rfloor$. Let $Q(h)$ be the maximum number of queries for $x$ in a tree of height $h$. We have the recurrence: $Q(h) \leq 1 + Q(\lfloor h/2 \rfloor)$. With the base case $Q(1) = 1$, this yields $Q(h) \leq 1 + \floor{\log_2(h)}$ as desired.
\end{proof}

\paragraph{Proof of \Cref{lem:error_generic}.}

\begin{proof}
    It follows from the monotonicity of the tree structure. The statement clearly holds for any nodes that are queried by $\textsc{CheckThreshold}$. 
    We have to show that it also holds for nodes that are marked without being directly queried.
    Such nodes that are marked \textit{Light} are descendants of some node $u$ where we had $\textsc{CheckThreshold}_{\tau}(\#count(u)) = \bot$. These nodes therefore have a count $< \tau + \beta$ since $u$ cannot have a lower count than any of it's descendant. Similarly, nodes that are marked \textit{Heavy} are ancestors of some node where $\textsc{CheckThreshold}$ returned $\top$ and thus must have a count of $> \tau - \alpha$.
\end{proof}

\paragraph{Proof of \Cref{lem:algorithm-error}.}

\begin{proof}
    Since $0 \leq \tau - \Delta - 1$ %$0 < \tau - \Delta$
    all empty nodes are correctly classified as $\textit{Light}$. 
    Let $m$ denote the number of calls to $\textsc{CheckThreshold}$ with non-zero counts performed by \Cref{alg:heavy_hitters}. 
    It follows from \cref{lem:bin_search_sensitivity} that $m \leq n\ceil{\log_2(h+1)}$. %$m \leq n(1 + \floor{\log_2(h)}) = n\ceil{\log_2(h+1)}$. 
    Let $(Z_1,\dots,Z_m)$ denote the noise samples for these queries.
    By a standard Gaussian tail bound we have $\Pr[\vert Z_i \vert \geq t \cdot \sigma] \leq 2 \exp(- t^2/2)$ and thus $\Pr[\vert Z_i \vert \geq \alpha] \leq \beta / m$.
    By applying a union bound we thus have $\Pr[\max_{i \in [m]}\vert Z_i \vert \geq \alpha] \leq \beta$.
\end{proof}

\paragraph{Proof of \Cref{lem:theorem-simpler}.}

\begin{proof}
    Here we use the looser add-the-deltas approach discussed in \cite{WilkinsKZK24}. The Gaussian noise is calibrated to satisfy $(\varepsilon, \delta/2)$-DP for $\ell_2$-sensitivity $\sqrt{m} = \sqrt{k(1 + \floor{\log_2 h})}$ using~\citep[Appendix~A]{DworkR14}. 
    We have that $\Pr[\mathcal N(0, \sigma^2) \geq \Delta] \leq \delta/(2m)$ which implies that the maximum noise for $m$ queries is less than $\Delta$ with probability at least $1 - \delta/2$. 
\end{proof}

\section{Privacy Guarantees of \Cref{alg:sparse_adaptive_queries}}
\label{app:sparse_gaussian}

\Cref{alg:sparse_adaptive_queries} is an adaptive variant of the Gaussian Sparse Histogram Mechanism from \citep{WilkinsKZK24}. The privacy properties from their work apply to this adaptive setting when the number of queries affected by any data point is bounded.

\begin{lemma}
    \label{lem:privacy-adaptive}
    Let $(q_1, \dots,q_d)$ be (adaptively chosen) counting queries such that for any $x \in \mathcal{X}$, we have $q_i(x) = 1$ for at most $m$ values of $i \in [d]$ and $q_j(x) = 0$ for all other queries. Then \cref{alg:sparse_adaptive_queries} with parameters $\sigma$ and $\Delta$ satisfies $(\varepsilon, \delta)$-differential privacy where $\gamma(j) = (m - j)\log \Phi \left(\frac{\Delta}{\sigma}\right)$ and
    \begin{align*}
        &\delta \geq \max \bigg[ 1 - \Phi \left( \frac{\Delta}{\sigma} \right)^m, \\
        &\max_{j \in [m]} 1 - \Phi \left( \frac{\Delta}{\sigma} \right)^{m - j} + \Phi \left( \frac{\Delta}{\sigma} \right)^{m - j} \left[ \Phi\left(\frac{\sqrt{j}}{2\sigma} - \frac{(\varepsilon - \gamma(j))\sigma}{\sqrt{j}}\right) - e^{\varepsilon - \gamma(j)} \Phi\left(-\frac{\sqrt{j}}{2\sigma} - \frac{\left(\varepsilon - \gamma(j)\right)\sigma}{\sqrt{j}}\right) \right], \\
        &\max_{j \in [m]} \Phi\left(\frac{\sqrt{j}}{2\sigma} - \frac{\left(\varepsilon + \gamma(j)\right)\sigma}{\sqrt{j}}\right) - e^{\varepsilon + \gamma(j)} \Phi\left(-\frac{\sqrt{j}}{2\sigma} - \frac{(\varepsilon + \gamma(j))\sigma}{\sqrt{j}}\right) \bigg]\,.
    \end{align*}
\end{lemma}

\begin{proof}
    It follows from the proof of \cite[Theorem~5.4]{WilkinsKZK24}. 
    \citet{WilkinsKZK24} study an algorithm in the non-adaptive setting, but the algorithm does not rely on any correlation between queries so the same analysis can be applied for adaptive queries. 
    The privacy guarantees follow from standard composition properties of the Gaussian mechanism combined with a bound of the probability of infinite privacy loss events for queries that have a value $< \tau - \Delta - 1$ for one dataset and $\geq \tau - \Delta - 1$ for the other.
    The proof implicitly requires that the queries are monotonic. 
    That is, all $d$ queries should be either non-increasing or non-decreasing between neighboring datasets.
    This assumption clearly holds for our setting since we, like \citep{WilkinsKZK24}, study counting queries.
\end{proof}

\section{Deterministic Error Bound for Heavy Hitter Algorithm}
\label{app:double_threshold}

Here we present a variant of our heavy hitters algorithm with deterministic error bound. We modify \Cref{alg:sparse_adaptive_queries} to always output $\top$ for queries that are sufficiently above the threshold.
We use the add-the-deltas technique discussed in \cite{WilkinsKZK24} to analyze the privacy guarantees.

\begin{algorithm}[h]
    \caption{GaussianSparseTwosidedThreshold (GSTT)}\label{alg:sparse_adaptive_queries_twosided}
    \begin{algorithmic}[1]
        \REQUIRE Parameters $\sigma$, $\tau$, and $\Delta$. 
        \REQUIRE Dataset $D$, (adaptive) counting queries $q_1, q_2, q_3, \dots$
        \FORALL{$i \in \{1,2,3,\dots\}$}
        \STATE Sample $Z_i \sim \N\left(0, \sigma^2\right)$.
        \IF{$q_i(D) \geq \tau + \Delta + 1$ or ($q_i(D) > \tau - \Delta - 1$ and $q_i(D) + Z_i > \tau$)}
        \STATE \textbf{Release} $\top$. 
        \ELSE
        \STATE \textbf{Release} $\bot$.
        \ENDIF
        \ENDFOR
    \end{algorithmic}
\end{algorithm}

\begin{lemma}
    \label{lem:privacy_double_threshold}
    Let $(q_1, \dots,q_d)$ be (adaptively chosen) counting queries such that for any $x \in \mathcal{X}$, we have $q_i(x) = 1$ for at most $m$ values of $i \in [d]$ and $q_j(x) = 0$ for all other queries. Then \cref{alg:sparse_adaptive_queries_twosided} with parameters $\sigma$ and $\Delta$ satisfies $(\varepsilon, \delta_{\text{Gauss}} + \delta_{\text{inf}})$-differential privacy where
    \begin{align*}
        \delta_{\text{Gauss}} &=  \Phi \left( \frac{\sqrt{m}}{2\sigma} - \frac{\varepsilon \sigma}{\sqrt{m}} \right) - e^\varepsilon \Phi \left( - \frac{\sqrt{m}}{2\sigma} - \frac{\varepsilon \sigma}{\sqrt{m}} \right), \quad  \text{and} \quad
        \delta_{\text{inf}} =  1 - \Phi \left( \frac{\Delta}{\sigma} \right)^{m} \,.
    \end{align*}
\end{lemma}

\begin{proof}
    By \Cref{def:differential-privacy}, we have to show that for any pair of neighboring datasets $D$ and $D'$ and all sets of outputs $Z$ we have
    \[
        \Pr[\textsc{GSTT}(q_1(D),\dots,q_d(D)) \in Z] \leq e^\varepsilon \Pr[\textsc{GSTT}(q_1(D'),\dots,q_d(D')) \in Z] + \delta_{\text{Gauss}} + \delta_{\text{inf}} \,.
    \]
    
    We prove that the inequality above by introducing a series of values $(\hat q_1,\dots,\hat q_d)$ that "lies between" $(q_1(D),\dots,q_d(D))$ and $(q_1(D'),\dots,q_d(D'))$.
    Consider the set of all possible counting queries $\mathcal{Q}$.
    Let $S \subseteq \mathcal{Q}$ denote the set of counting queries that are \emph{not} classified deterministically for $D$. That is, $q(D) \in [\tau - \Delta - 1, \tau + \Delta + 1]$ if and only if $q \in S$ and define similarly $S'$ and $\hat S$.
    We construct the intermediate values such that $S' = \hat S$. 
    For all $q_i \in (S \cap S')$ we set $\hat q_i = q_i(D)$ and for all $q_i \notin (S \cap S')$ we set $\hat q_i = q_i(D')$. 
    In other words, we construct $(\hat{q}_1,\dots,\hat{q}_d)$ such that 
    (1) $(q_1(D'),\dots,q_d(D'))$ and $(\hat{q}_1,\dots,\hat{q}_d)$ only differ in entries that are in both $S$ and $S'$
    (2) $(q_1(D),\dots,q_d(D))$ and $(\hat{q}_1,\dots,\hat{q}_d)$ only differ in entries that are in only one of $S$ and $S'$.
    This allows us to analyze each case separately to derive our values for $\delta_{\text{Gauss}}$ and $\delta_{\text{inf}}$.
    Note that the values $(\hat{q}_1,\dots,\hat{q}_d)$ might not correspond to evaluating counting queries on any existing dataset $\hat D$. 
    We just use the values only for the proof and utilize the fact that \Cref{alg:sparse_adaptive_queries_twosided} only access the count.
    Notice that we cannot have any query where $q_i(D) < \tau - \Delta - 1$ and $q_i(D') > \tau + \Delta + 1$ or vice versa. 
    As such, whenever $q_i \notin (S \cup S')$ the deterministic output of \Cref{alg:sparse_adaptive_queries_twosided} is the same for both datasets. Therefore we only have to consider queries that are deterministic for exactly one or for none of the datasets. The intermediate values allow us to analyze the two cases separately.

    First we show that
    \[
    \Pr[\textsc{GSTT}(q_1(D),\dots,q_d(D)) \in Z] \leq \Pr[\textsc{GSTT}(\hat{q}_1,\dots,\hat{q}_d) \in Z] + \delta_{\text{inf}} \,.
    \]
    Since $q_i(D) \neq \hat q_i$ for at most $m$ queries any change in output distribution is due to those queries. For each such query, one of the outputs is deterministic. Without loss of generality, assume that $q_i(D) < \tau - \Delta - 1$ such that the output is always $\bot$. The probability that the output is $\bot$ with $\hat q_i$ is at least $\Phi \left( \frac{\Delta}{\sigma} \right)$, since $Z_i$ must be larger than $\Delta$ to change the output to $\top$.
    The probability that all the $m$ queries are identical is therefore at least $\Phi \left( \frac{\Delta}{\sigma} \right)^m$ which gives the bound for $\delta_{\text{inf}}$.

    Next consider the case of $(\hat{q}_1,\dots,\hat{q}_d)$ and $(q_1(D'),\dots,q_d(D'))$. These inputs only differ for queries in $[\tau -\Delta - 1, \tau + \Delta + 1]$. There are at most $m$ differing queries, so the $\ell_2$-distance is bound by $\sqrt{m}$. The algorithm can be seen as running the Gaussian mechanism on values in $[\tau -\Delta - 1, \tau + \Delta + 1]$ and post-processing the values. The tight bound for the Gaussian mechanism (\Cref{lem:gaussian-mech}) gives us 
    \[
    \Pr[\textsc{GSTT}(\hat{q}_1,\dots,\hat{q}_d) \in Z] \leq e^\varepsilon\Pr[\textsc{GSTT}(q_1(D'),\dots,q_d(D')) \in Z] + \delta_{\text{Gauss}} \,.
    \]
    Note that we technically cannot apply \Cref{lem:gaussian-mech} directly, since the Gaussian mechanism is typically analyzed in the static setting, while we allow adaptively chosen queries.
    A more rigorous proof would fix the neighboring datasets and study the input as a series of adaptive queries rather than a fixed sequence. 
    Under Gaussian Differential Privacy~\citep{Dong22GDP} we (informally) allow an analyst to make $m$ queries that are each $1/\sigma$-GDP and an unlimited number of $0$-GDP queries (those where the counts are the same in our original setup). 
    Using the adaptive composition property and conversion from $\mu$-GDP to $(\varepsilon,\delta)$-DP recovers the inequality above. We leave the details as an exercise for the reader.

    Combining the two inequalities above give us the desired bound 
    \begin{align*}
        \Pr[\textsc{GSTT}(q_1(D),\dots,q_d(D)) \in Z] &\leq \Pr[\textsc{GSTT}(\hat{q}_1,\dots,\hat{q}_d) \in Z] + \delta_{\text{inf}} \\ 
        &\leq e^\varepsilon \Pr[\textsc{GSTT}(q_1(D'),\dots,q_d(D')) \in Z] + \delta_{\text{Gauss}}  + \delta_{\text{inf}} \,.
    \end{align*}
\end{proof}

Similarly to \Cref{lem:theorem-simpler}, we give a simple upper bound for the values of $\sigma$ and $\Delta$.

\begin{corollary}
    \label{cor:deterministic-simpler}
    When $\varepsilon < 1$ then the parameters 
    \[
        \sigma = \frac{\sqrt{2m\ln(2.5/\delta)}}{\varepsilon} \quad \quad \quad
        \Delta = \sigma \cdot \sqrt{2\ln(2 m/\delta)} = O\left(\frac{\sqrt{m}\log(1/\delta)}{\varepsilon}\right)
    \]
    satisfy the condition in \Cref{lem:privacy_double_threshold} for $(\varepsilon, \delta)$-DP. In the last equality we assume that $m \ll 1/\delta$.
\end{corollary}

Since \Cref{alg:sparse_adaptive_queries_twosided} clearly can only misclassify queries in the interval $[\tau - \Delta - 1, \tau + \Delta + 1]$, it follows that we achieve a deterministic error bound for the heavy hitters problem.

\begin{corollary}
    \label{cor:deterministic-tree-error}
    Let $\mathcal{A}$ denote an algorithm that takes as input $k$ trees with maximum height $h$. The algorithm runs \Cref{alg:heavy_hitters} for each tree where $\textsc{CheckThreshold}$ is instantiated as \Cref{alg:sparse_adaptive_queries_twosided} with parameters $\sigma$, $\tau$, and $\Delta$. Then $\mathcal{A}$ satisfies $(\varepsilon,\delta)$-DP with values specified by \Cref{lem:privacy_double_threshold}. 
    Let $\alpha \in \mathbb{R}$ denote the smallest value such that all nodes with at least count $\tau + \alpha$ are marked as \textit{Heavy} and all nodes with at most count $\tau - \alpha$ are marked \textit{Light}.
    Then it holds with probability 1 that $\alpha \leq 1 + \Delta$.
    Furthermore, when $\tau \geq 1 + \Delta$ we can implement the algorithm to run in time and space $O(knh)$. \looseness=-1
\end{corollary}

We note that it is possible to get tighter results using a case by case analysis similar to \Cref{lem:privacy-adaptive}.

\section{Proof of Lower Bound~(\Cref{lem:informal-lower-bound})}
\label{app:lower_bound}

Here we discuss the lower bound for deep trees presented in \Cref{lem:informal-lower-bound}.
We use the lower bound of \cite{BalcerV19} which we restate below. Note that Balcer and Vadhan use a different definition of differential privacy.
They study the problem under the replacement definition for neighboring datasets where the size of the dataset is known (bounded DP) while we use the add/remove (unbounded DP) definition.
However, we can still use their result since any DP mechanism under add/remove satisfies DP with slightly worse parameters under bounded DP by a simple group privacy argument.

\begin{lemma}
    \label{lem:add-remove-to-replacement}
    A heavy hitters algorithm $\mathcal{M}$ that satisfies $(\varepsilon,\delta)$-DP
    under the add/remove neighboring relation satisfies $(2\varepsilon,(1 + e^\varepsilon)\delta)$-DP under the replacement definition.
\end{lemma}

\begin{proof}
    This is well known in differential privacy. We include a short proof for completeness. Let $i$ be the index that differs between some pair of neighboring datasets $D\sim D'$ under replacement. That is, $D_j = D'_j$ for all $j \neq i$. 
    Then define an intermediate dataset $\hat D = D \setminus \{D_i\} = D' \setminus \{D'_i\}$. We have
    \begin{align*}
        \Pr[\mathcal{M}(D) \in Z] & \leq e^\varepsilon \Pr[\mathcal{M}(\hat D) \in Z] + \delta \\
        & \leq e^\varepsilon (e^\varepsilon\Pr[\mathcal{M}(D') \in Z] + \delta) + \delta \\
        & = e^{2\varepsilon} \Pr[\mathcal{M}(D') \in Z] + (1 + e^\varepsilon) \delta \,,
    \end{align*}
    where the inequalities follow from the fact that $D\sim\hat D$ and $\hat D \sim D'$ are pairs of neighboring datasets under the add/remove definition for which $\mathcal{M}$ satisfies $(\varepsilon,\delta)$-DP.
\end{proof}

\begin{lemma}[Theorem~7.2 of \cite{BalcerV19}]
    \label{app:lem:BV19lower}
    Let $\mathcal{M}$ be a $(\varepsilon, \delta)$-differentially private algorithm (under bounded DP) that takes as input any dataset of size $n$ and $d$ counting queries $(q_1(D), \dots, q_d(D))$ where each data point satisfies exactly one counting query. 
    Assume that $\mathcal{M}$ has the following error guarantee
    \[
        \forall D \in \mathcal{X}^n \quad \forall i \in [d] \quad \Pr[\vert \mathcal{M}(D)_i - q_i(D) \vert \leq \alpha] \geq 1 - \beta
    \]
    with $\beta \in (0, 1/2]$.
    If $\mathcal{M}$ outputs at most $\hat{n}$ non-empty bins.
    Then 
    \[
        \alpha \geq \frac{1}{2} \cdot \min \left(\frac{1}{2\varepsilon} \ln\left(\frac{d}{16 \beta \hat{n}}\right) - 1, \frac{1}{\varepsilon} \ln\left(\frac{\varepsilon}{4\delta}\right) - 1, n\right) \,.
    \]
\end{lemma}

Balcer and Vadhan refer to the type of error guarantee from the above lemma as the $(\alpha,\beta)$-Per-Query error. We use the following statement in the proof.

\begin{lemma}[{\cite[Figure~3]{BalcerV19}}]
    \label{lem:per-query-error}
    The mechanism that adds Geometric noise to all counting queries satisfies $(\alpha,\beta)$-Per-Query accuracy (under bounded DP) for any 
    \[
        \alpha \geq \ceil{2/\varepsilon \cdot \ln(1/\beta)} \,.
    \]
\end{lemma}

Now we restate an exact version of our informal lower bound \Cref{lem:informal-lower-bound}.
Note that we do not optimize constants. For simplicity, we state the result only for $\varepsilon < 1$. 
This restriction is only used to easily deal with the $e^\varepsilon$ part of the $(1 + e^\varepsilon)\delta$ term that arises from changing DP definitions. 
For a tighter analysis, one should directly analyze our mechanism under replacement DP or carry out the lower bound proof of \cite{BalcerV19} under add/remove DP.

\begin{lemma}
    \label{app:lem:lower-bound}
    Let $\mathcal{M}$ denote an $(\varepsilon, \delta)$-DP mechanism (for some $\varepsilon < 1$) that takes as input a tree $\T$ of height $h > 6 + \log_2(n^2/\beta^{15})$ and a dataset of size $n$ and outputs a set of at most $n^2$ nodes per layer of $\T$. Suppose that with probability at least $1 - \beta$, $\mathcal{M}$ outputs all nodes in $\T$ with count above some value $\hat\tau \in \mathbb{R}$. 
    Then 
    \[
        \hat \tau \geq \frac{1}{2} \cdot \min \left(\frac{1}{8\varepsilon} \ln\left(\frac{2^{h-1}}{32 \beta n^2}\right) - 1, \frac{1}{4\varepsilon} \ln\left(\frac{\varepsilon}{4\delta}\right) - 1, n\right) \,.
    \].
\end{lemma}

\begin{proof}
    We can encode any sparse histogram~(see \cref{app:lem:BV19lower}) with $d = 2^{h-1}$ queries in the leaves of the tree.
    Label the leaves in arbitrary order and set the count of the first leaf to $q_1(D)$, the count of the second leaf to $q_2(D)$ etc. 
    By \Cref{lem:add-remove-to-replacement}, $\mathcal{M}$ satisfies $(2\varepsilon,(1+e^\varepsilon)\delta)<(2\varepsilon,4\delta)$-DP in the setting of \cite{BalcerV19}.
    We can spend additional privacy budget to add Geometric noise calibrated for $(2\varepsilon,0)$-DP to the all leaves of $\T$ that are part of the output of $\mathcal{M}$. This clearly does not violate the requirement of outputting at most $\hat n = n^2$ non-empty bins. 
    For any leaf that was part of the output of $\mathcal{M}$, the $(\alpha,\beta)$ accuracy is $\ceil{1/\varepsilon \cdot \ln(1/\beta)}$ by \cref{lem:per-query-error}. 
    Thus all queries above $\hat \tau$ have this error with failure probability at most $2\beta$.
    By our assumption on $h$ the error of nodes above $\hat \tau$ is lower than the error required by the lower bound.
    The leaves with value below $\hat \tau$ must therefore increase the $(\alpha, \beta)$-Per-Query error.
    We have no requirements on $\mathcal{M}$ for nodes with count less than $\hat \tau$. 
    If we simply output an estimate of zero for all nodes not in the output of $\mathcal{M}$, these leaves can therefore have expected error up to $\hat \tau$.
    Since the mechanism above is $(4\varepsilon,4\delta)$-DP and has failure probability at most $2\beta$ by \Cref{app:lem:BV19lower} we must have
    \[
        \hat \tau \geq \frac{1}{2} \cdot \min \left(\frac{1}{8\varepsilon} \ln\left(\frac{d}{32 \beta n^2}\right) - 1, \frac{1}{4\varepsilon} \ln\left(\frac{\varepsilon}{4\delta}\right) - 1, n\right) \,.
    \]
\end{proof}

%\cl{Note for self. We need the results of Theorem 7.2 here. And likely some condition such as $h > \log(n)$. Without Theorem 7.2 we can only discuss the maximum error, and then we have to deal with bounds on the the $O(n^2)$ samples. Theorem 7.2 is per-query error so we can just use half of $\varepsilon$ for Laplace noise.}

\section{Implementation Details}
\label{app:implementation}

\textbf{Tree construction.}
In this section we present implementation details of our algorithm. In particular, we discuss how to implement \Cref{alg:heavy_hitters} to run in time and space $O(\ell)$, where $\ell$ is the number of nodes with true count at least $\tau - \Delta - 1$. 
Since each data point increments the count for a single node in each layer of the tree we clearly have $\ell = O(nh)$ as claimed in \Cref{thm:build-forest}.
Note that in the pseudocode we explicitly mark all nodes either \textit{Heavy} or \textit{Light} which is infeasible for large trees. In practice, we implement the algorithm to return the set of nodes marked \textit{Heavy}. Any nodes that are not explicitly marked by the algorithm are not returned as they are implicitly \textit{Light}.

First, we briefly discuss how we build the unpruned decision tree in time $O(nh)$ starting with the root. We assume that we always select a splitting feature and threshold in time $O(1)$.
We then split a node with $m$ points in time $O(m)$ using a simple pass over the nodes to classify them for the left or right child node.
Since each data point is part of exactly one split in each layer, the total running time is $O(nh)$.
We do not split empty nodes since all nodes in empty subtrees are eventually marked \textit{Light} when $\tau > 1 + \Delta$. If the threshold is larger, we can similarly stop splitting nodes with less than $\tau - \Delta - 1$ data points.

Next, we show that each part of \Cref{alg:heavy_hitters} can be implemented in linear running time in the size of the tree $O(\ell)$.
In Phase 1 we query all nodes in $L_{\text{mid}}$ and mark ancestors or descendants. 
Marking ancestor of a node takes up to time $O(h)$. However, we stop early when we encounter a node which was already marked, so in total over the entire run of the algorithm this takes at most time $O(\ell)$.
We do not actually need to explicitly mark all descendants of \textit{Light} nodes, since we skip those recursive calls in Phase 2. 
After we finish marking nodes in the tree we find the set of heavy hitters by adding nodes starting from the root. Since we stop exploring any path when we encounter a \textit{Light} node, we never explore the unmarked subtrees.

Now, notice that we call $\textsc{CheckThreshold}_\tau(\#count(u))$ for any node $u$ at most once for the entire execution of all recursive calls. 
If we can access all non-zero nodes in $L_{\text{mid}}$ in time $O(\vert L_{\text{mid}} \vert)$ the entire algorithm therefore runs in $O(\ell)$ (We assume that we sample a Gaussian in time $O(1)$).
However, simply traversing the tree from the root to find $L_{\text{mid}}$ takes time up to $O(h \vert L_{\text{mid}}\vert )$. 
We reduce the running time with some algorithmic engineering.

We label the nodes (including missing empty nodes) in top-down and left-to-right order. The root has label $1$, the layer below the root has labels $2$ and $3$, the next layer contains labels $4$ to $7$ etc.
For any node with label $i$, the left child has label $2 \cdot i$ and the right child has label $2 \cdot i + 1$. This is a standard technique for storing a binary tree in an array without using pointers. 
This labeling is convenient, because when the top node in $\mathcal{T}$ has label $i$, then $L_{\text{mid}}$ are exactly the nodes with labels from $i \cdot 2^{\text{mid}}$ to $(i + 1) \cdot 2^{\text{mid}} -1$, inclusive.
However, for large trees most of these nodes do not exist, so we need to efficiently find all such nodes without explicitly querying each label. Ideally, we do not want to use any advanced data structure to avoid adding unnecessary complexity.

Notice that our algorithm always explores the tree in left-to-right order. 
In particular, if we recursively call the algorithm in ascending label order in Phase 2, then for any particular layer we query the nodes in ascending label order.
As such, when we call $\textsc{CheckThreshold}_\tau(\#count(u))$ we know that no recursive call will ever need to query $\textsc{CheckThreshold}_\tau(\#count(v))$ for any $v$ where $u$ and $v$ has the same depth in the full tree and $\text{label}(v) < \text{label}(u)$.
We utilize this observation to achieve an efficient implementation. We store a queue for each layer in the tree and add all nodes in the tree to the corresponding queue using a single traversal of the tree. 
When we need to access $L_{\text{mid}}$ in a recursive call where the root of the subtree has label $i$ we discard all nodes in the corresponding queue with label $< i \cdot 2^{\text{mid}}$.
We then query all nodes from the queue with label $< (i + 1) \cdot 2^{\text{mid}}$.
Those are exactly the non-empty nodes in $L_{\text{mid}}$ and we query them in left-to-right order as desired.

\textbf{Leaf algorithm.}
We now briefly discuss the privacy analysis of the two leaf algorithms used in our implementation. 
The techniques we use are standard. We use the exponential mechanism for majority vote and Gaussian noise for approximate class distributions. We use zCDP for accounting in both cases.  
For the class distribution we simply add noise from $\mathcal{N}(0,k/(2\rho))$ to each class count. This can create negative counts, but we can clip to non-negative values as post-processing. 

\cite{DPorg-exponential-mechanism-bounded-range} show that the exponential mechanism with privacy parameter $\varepsilon$ satisfies $\frac{1}{8}\varepsilon^2$-zCDP by showing it satisfies a condition known as bounded range $\varepsilon$-bounded range. %\de{Do we need to introduce bounded range or it's conversion to epsilon DP in the preliminaries? It seems very specific.} 
Moreover, in the proof for Lemma 3 in \citep{DPorg-exponential-mechanism-bounded-range}, they make use of the bound
\begin{align*}
    - \Delta \le \ell(y, x') - \ell(y,x) \le \Delta 
\end{align*}
which can be improved in our case, because the scoring function - being the count of each class - is monotonic. Depending on whether we add or remove a row from the dataset, the term $\ell(y, x') - \ell(y,x)$ is in $[0, \Delta]$ or $[-\Delta, 0]$. It follows that the mechanism satisfies $\varepsilon/2$-bounded range and $\frac{1}{32}\varepsilon^2$-zCDP. Monotonicity therefore makes it sufficient to set the $\varepsilon$ parameter of the standard exponential mechanism~\cite{McSherryT07} to $\frac{1}{\sqrt{8\rho / k}}$ instead of $\frac{2}{\sqrt{8\rho / k}}$. 

%\de{I probably have to explain which noise scale? But then I would also have to explain that there is an equivalence between EM and noisy max with Gumbel?}
%\cl{I think we can keep it very high level, since the exponential mechanism is standard. I don't know the reference for the monotonicity claim. The only reference I know for zCDP is the post on differentialprivacy.org. You write in the source code that we match opendp. They must have a reference we can use.}

\section{Regression Trees}
\label{app:regression}

%\update{This entire section is new}

In the main paper, we focused on classification problems, in line with prior work. In this section, we discuss how to extend our approach to regression trees. The tree construction algorithm is the same; the only difference is the algorithm used for privatizing leaves. 
In regression, data points no longer have a discrete label $y \in \mathcal{Y}$, but instead have a numeric one $y \in \mathbb{R}$. We assume that $y$ is bounded to some known interval $[l,u]$. 
We label each leaf with a private estimate of the label mean of data points in the leaf. Specifically, we first center the queries such that $\hat y = y - \tfrac{u - l}{2}$, and privately estimate the sum of shifted values, whose sensitivity is bounded by $\tfrac{u - l}{2}$. We then divide the noisy sum by a noisy estimate of the count (clipped to be at least $1$), and project the private estimate back to the original range. We split the privacy budget equally between the two queries. This is a standard approach for mean estimation in the add/remove setting; see Algorithm 2 in \cite{KuleszaSW24} for more details (note that we use Gaussian noise whereas they use Laplace noise). 

We evaluate our mechanism on two datasets. \texttt{MedicalCharges} (OpenML ID 45048~\url{https://www.openml.org/search?type=data&sort=runs&id=45048}) has 163{,}065 records, 3 numerical features, and the target is the logarithm of the average total medical expenses.
\texttt{CaliforniaHousing}~\citep{pace1997sparse} has 20{,}640 records (one per census block group in the 1990 U.S.\ census), 8 numerical features, and the target is the median house value in units of \$100{,}000, capped at 5.

We fit 30 trees and use the same hyperparameters as in \Cref{sec:experiments}, except that we allocate more privacy budget for leaf queries by setting $\varepsilon_1=0.25\varepsilon$. We compare against Sklearn Extra Tree Regressor (50 trees) and a Sklearn Decision Tree Regressor. In \Cref{fig:appendix-regression} we report $R^2$ value and RMSE (Root Mean Square Error). For the Extra Tree models we report results for two splitting threshold. Each experiment is rerun 5 times.

These preliminary results show that our technique supports regression trees. However, the performance is worse than in the classification setting: we need larger privacy budgets to approach the performance of the non-private baselines. One of the key challenges for regression is that the noisy mean has large variance for small leaves, which can significantly affect the aggregate prediction. This could be somewhat mitigated by ensuring that leaves have sufficiently many data points to provide an accurate mean estimate. However, this limits the trees' ability to fit the data. The non-private baseline shows that performance can degrade with higher splitting threshold. This suggests that high utility for regression may require large datasets. We leave the further exploration of private regression forests to future work.

\begin{figure}[h]
    \centering
    \includegraphics[width=0.95\linewidth]{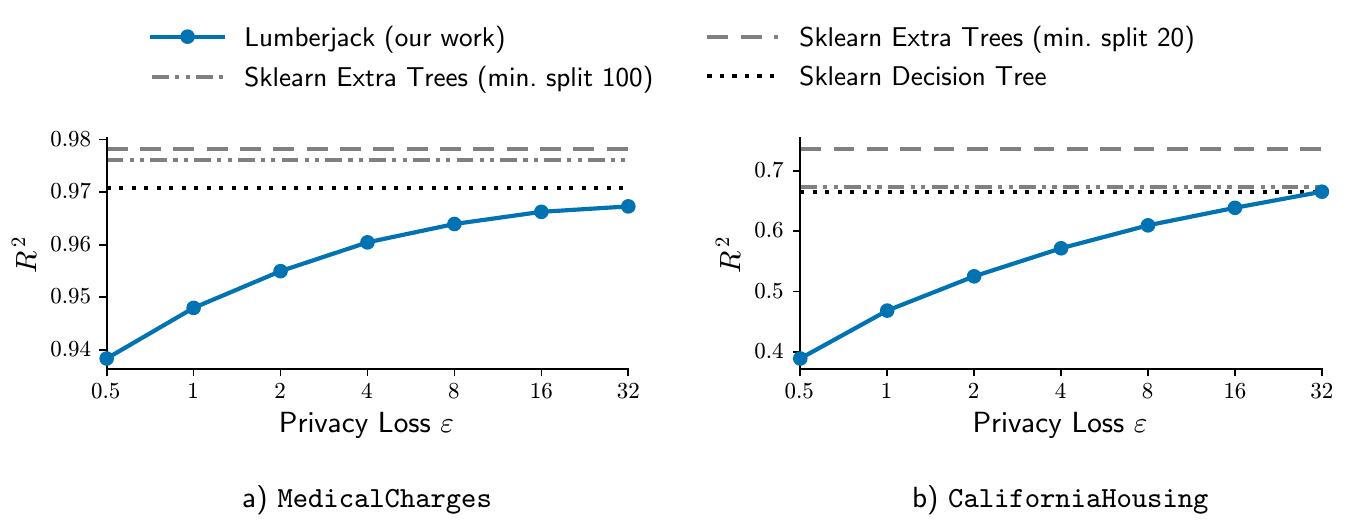}
    \includegraphics[width=0.95\linewidth]{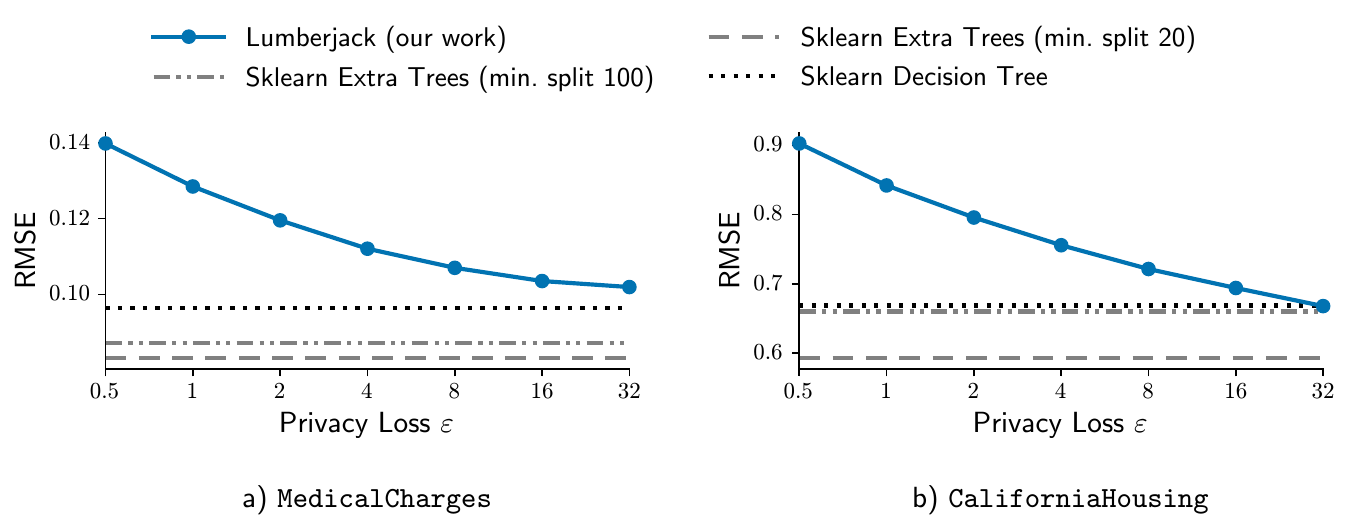}
    \caption{Experimental results for regression datasets. Since prior work does not consider regression we only compare with non-private baselines. 
    }
    \label{fig:appendix-regression}
\end{figure}

\end{document}